\documentclass[sigconf]{aamas} 
\usepackage{balance} % for balancing columns on the final page

\usepackage{amsmath}
\usepackage{mathtools}
\usepackage{amsfonts}
\usepackage{wasysym}
\usepackage{thmtools}
\usepackage{thm-restate}
\usepackage{algorithmic}
\usepackage[ruled,vlined]{algorithm2e}
\usepackage{mathbbol}

\usepackage{caption}
\usepackage{subcaption}
\usepackage{hyperref}
\usepackage{wrapfig}

\setcopyright{ifaamas}
\acmConference[AAMAS '22]{Proc.\@ of the 21st International Conference
on Autonomous Agents and Multiagent Systems (AAMAS 2022)}{May 9--13, 2022}
{Online}{P.~Faliszewski, V.~Mascardi, C.~Pelachaud,
M.E.~Taylor (eds.)}
\copyrightyear{2022}
\acmYear{2022}
\acmDOI{}
\acmPrice{}
\acmISBN{}

\acmSubmissionID{717}

\title[Streaming Bandits]{Efficient Algorithms for Finite Horizon and Streaming Restless Multi-Armed Bandit Problems}

\author{Aditya Mate}
\affiliation{
  \institution{Harvard University}
  \country{Boston, MA, U.S.A.}
 }
\email{aditya\_mate@g.harvard.edu}

\author{Arpita Biswas}
\affiliation{
  \institution{Harvard University}
  \country{Boston, MA, U.S.A.}
 }
\email{arpitabiswas@seas.harvard.edu}

\author{Christoph Siebenbrunner}
\affiliation{
  \institution{Harvard University}
  \country{Boston, MA, U.S.A.}
 }
\email{csiebenbrunner@seas.harvard.edu}

\author{Susobhan Ghosh}
\affiliation{
  \institution{Harvard University}
  \country{Boston, MA, U.S.A.}
 }
\email{susobhan\_ghosh@g.harvard.edu}

\author{Milind Tambe}
\affiliation{
  \institution{Harvard University}
  \country{Boston, MA, U.S.A.}
 }
\email{milind\_tambe@harvard.edu}

\begin{abstract}
We propose Streaming Bandits, a Restless Multi-Armed Bandit (RMAB) framework in which heterogeneous arms may arrive and leave the system after staying on for a finite lifetime. Streaming Bandits naturally capture the health-intervention planning problem, where health workers must manage the health outcomes of a patient cohort while new patients join and existing patients leave the cohort each day. Our contributions are as follows: (1) We derive conditions under which our problem satisfies indexability, a pre-condition that guarantees the existence and asymptotic optimality of the Whittle Index solution for RMABs. We establish the conditions using a polytime reduction of the Streaming Bandit setup to regular RMABs. (2) We further prove a phenomenon that we call index decay — whereby the Whittle index values are low for short residual lifetimes — driving the intuition underpinning our algorithm. (3) We propose a novel and efficient algorithm to compute the index-based solution for Streaming Bandits. Unlike previous methods, our algorithm does not rely on solving the costly finite horizon problem on each arm of the RMAB, thereby lowering the computational complexity compared to existing methods. (4) Finally, we evaluate our approach via simulations run on real-world data sets from a tuberculosis patient monitoring task and an intervention planning task for improving maternal healthcare, in addition to other synthetic domains. Across the board, our algorithm achieves a 2-orders-of-magnitude speed-up over existing methods while maintaining the same solution quality. The full paper is available at: \url{https://arxiv.org/pdf/2103.04730.pdf}
\end{abstract}

\keywords{Restless Multi-Armed Bandits;
Whittle Index;
Finite Horizon;
Intervention Planning}

\newtheorem{theorem}{Theorem}

\newtheorem{fact}{Fact}

\newtheorem{lemma}{Lemma}
\newtheorem{definition}{Definition}

\usepackage{chngcntr,etoolbox}
\newcounter{resetdummycounter}
\newcommand{\resetcounterlist}[1]{%
  \renewcommand*{\do}[1]{\counterwithin*{##1}{resetdummycounter}}%
  \docsvlist{#1}}

\resetcounterlist{theorem, lemma}

\usepackage{soul} % \hl command
\soulregister\cite7 % required so that \hl also works with citations,...
\soulregister\ref7 % https://tex.stackexchange.com/questions/139463/how-to-make-hl-highlighting-to-automatically-place-incompatible-commands-in/139500
\soulregister\eqref7
\soulregister\pageref7

\begin{document}

%%% The following commands remove the headers in your paper. For final 
%%% papers, these will be inserted during the pagination process.

\pagestyle{fancy}
\fancyhead{}

%%% The next command prints the information defined in the preamble.

\maketitle 

%%%%%%%%%%%%%%%%%%%%%%%%%%%%%%%%%%%%%%%%%%%%%%%%%%%%%%%%%%%%%%%%%%%%%%%%

\section{Introduction}

% \label{sec:intro}
% \begin{figure}
% \centering
% \includegraphics[width=0.40\textwidth]{figs/summary.png}
% \caption{We summarize our con }%
% \label{fig:summary}
% \end{figure}

\label{sec:intro}
% \begin{figure}
% \centering
% \includegraphics[width=0.40\textwidth]{figs/summary2.png}
% \caption{For the S-RMAB setting, existing approaches are either accurate but slow or fast but inaccurate. Our proposed approach combines the benefits to lead to fast, accurate algorithm for S-RMABs. \hl{didn't we say let's change performance to reward?}
% }%
% \label{fig:summary}
% \end{figure}

In community healthcare settings, adherence of patients to prescribed health programs, that may involve taking regular medication or periodic health checkups, is critical to their well-being. One way to improve patients' health outcomes is by tracking their health or monitoring their adherence to such programs. Such health monitoring programs combined with suitably designed intervention schemes help patients alleviate health issues such as diabetes~\citep{newman2018}, hypertension~\citep{brownstein2007}, tuberculosis~\citep{chang2013,rahedi2014}, depression~\citep{lowe2004,mundorf2018}, etc. However, health interventions often require dedicated time of healthcare workers, which is a severely scarce resource, grossly inadequate to meet the total demand. This issue is especially more severe in the global south. Moreover, planning interventions with these limited resources is made more challenging due to the fact that the extent of adherence of patients may be both, uncertain as well as transient. Consequently, the healthcare workers have to grapple with this sequential decision making problem of deciding which patients to intervene on, with limited resources, in an uncertain environment. Existing literature on healthcare monitoring and intervention planning (HMIP)~\citep{Akbarzadeh2019,mate2020collapsing,bhattacharya2018restless, mate2021riskAware, mate2021field} casts this as a \emph{restless multi-armed bandit} (RMAB) planning problem. In this setup, the patients are typically represented by the arms of the bandit and the planner must decide which arms to pull (which patients to intervene on) under a limited budget. The RMAB problem formalizes the (restless) behavioral dynamics of the patients both in the presence and in the absence of interventions. 

%In this problem formulation, each patient's behavior (e.g., the extent of adherence) is assumed to transition according to a Markov Decision Process (MDP). One well-studied solution to the RMAB planning problem is called \emph{Whittle Index policy}~\cite{whittle1988restless}, and has been extensively used for solving  various subclasses of the RMAB problem. 
In addition to healthcare, RMABs have caught traction as solution techniques in a myriad of other domains involving limited resource planning for applications such as anti-poaching patrol planning \citep{qian2016restless}, multi-channel communication systems~\citep{zhao_liyu_paper}, sensor monitoring tasks~\citep{Glazebrook2006}, UAV routing~\citep{le2008multi} etc. For ease of presentation, we consider the HMIP problem for motivation but our approach is relevant and can be extended to other real-world domains.

%Solving for the optimal policy in RMABs has been shown to be PSPACE hard in general \citep{papadimitriou1999complexity}. So 
% A popular approach, known as the Whittle Index policy \citep{whittle1988restless}, has been shown to work very well, especially under an \textit{indexability} assumption. 
The existing literature on RMABs for intervention planning, however, has mainly focused on problems involving an infinite time horizon (i.e., the health programs are assumed to run forever) and, moreover, the results are limited to settings where no new patients (or bandit arms) arrive midway during the health program. We consider a general class of RMABs, which we call \emph{streaming restless multi-armed bandits}, or S-RMAB. In an S-RMAB instance, the arms of the bandit are allowed to arrive asynchronously, that is, the planner observes an incoming and outgoing stream of bandit arms. The classic RMAB (both with infinite and finite horizon) is a special case of the S-RMAB where all arms appear (leave) at the same time. Additionally, each arm of an S-RMAB is allowed to have its own transition probabilities, capturing the potentially heterogeneous nature of patient cohorts. S-RMABs display a special structure in the presence of streaming arms and a finite horizon, which the existing methods fail to utilize. Our approach exploits this structure to arrive at approaches that perform better in the streaming bandit setting.

A fairly general approach, proposed by \cite{qian2016restless} may be applied even when patients arrive and leave asynchronously after staying for a finite duration. The method allows to approximate the exact solution arbitrarily well, but it is computationally expensive as the number of patients or arms increases. A more recent approach, proposed by~\cite{mate2020collapsing}, exploits the structure of the HMIP and is considerably faster, but the method relies on the assumption of an infinite planning horizon. This algorithm suffers a severe deterioration in performance when employed on shorter horizon settings. 

Our \textbf{contribution} consists of proposing a new approach, designed for the finite-horizon and asynchronous arrival settings, %that uses the special structure of these settings to arrive at an algorithm 
that achieves a combination of the advantages of existing %unspecialized 
methods, i.e. high solution quality and low runtime, in those settings. We provide theoretical justifications for the use of Whittle indices in streaming RMABs, as well as for the setup of our algorithms, designed to leverage the structure of the finite horizon and asynchronous cases. We further show that our method also applies to S-RMAB arms exhibiting \emph{reverse threshold optimality}, while previous methods only applied to settings with \emph{forward threshold optimality}. We perform experimental evaluations of our algorithms using real-world data from two domains, as well as synthetic and adversarial domains. Our algorithms provide a 2-orders-of-magnitude speed-up compared to existing accurate methods, without loss in performance. 

\section{Related work} 
The RMAB problem was introduced by \cite{whittle1988restless}. The paper studied the RMAB problem with the goal of maximizing the average reward in a dynamic programming framework. Whittle formulated a relaxation of the problem and provided a heuristic called \emph{Whittle Index policy}. This policy is optimal when the underlying Markov Decision Processes satisfy {indexability}, which is computationally intensive to verify. Later, \cite{papadimitriou1994complexity} established that solving RMAB is PSPACE hard, even when the transition rules are known. Since then, specific classes of RMABs have been studied extensively. \cite{qian2016restless} studied the infinite horizon RMAB problem and proposed a binary search based algorithm to find
Whittle index policy. However, the algorithm becomes computationally expensive as the number of arms grows. \cite{bhattacharya2018restless} models the problem of maximizing health information coverage as an RMAB problem and proposes a hierarchical policy which leverages the structural assumptions of the RMAB model.  %\cite{lee2019optimal} study the problem of selecting patients for early-stage cancer screening, by formulating it as a subclass of RMAB.
\cite{Akbarzadeh2019} provide a solution for the class of bandits with ``controlled restarts'' and state-independent policies, possessing the indexability property. \cite{mate2020collapsing} model a health intervention problem, assuming that the uncertainty about the state collapses when an intervention is provided. They provide an algorithm called \emph{Threshold Whittle} to compute the Whittle indices for infinite horizon RMAB. There are many other papers that provide Whittle indexability results for different subclasses of Partially Observable Markovian bandits~\citep{Glazebrook2006,Hsu2018,Sombabu2020,Liu2010}. However, these papers focus on infinite horizon, whereas we focus on the more challenging setting when there is a fixed finite horizon. 

The RMAB problem with finite horizon has been comparatively less studied.  \cite{nino2011computing} provided solutions to the  one-armed restless bandit problem, where only one arm is activated at each time before a time horizon $T$. Their solution do not directly extend to the scenario when multiple arms can be pulled at each time step. \cite{hu2017asymptotically} considered finite horizon multi-armed restless bandits with identically distributed arms. They show that an index based policy based on the Lagrangian relaxation of the RMAB problem, similar to the infinite horizon setting, provides a near-optimal solution. \cite{lee2019optimal} study the problem of selecting patients for early-stage cancer screening, by formulating it as a very restricted subclass of RMAB. All these works consider that all the arms are available throughout $T$ time steps. Some other works, such as \citep{meuleau1998solving, hawkins2003langrangian} also adopt different approaches to decomposing the bandit arms, which may be applicable to finite horizon RMABs. These techniques to solving weakly coupled Markov Decision Processes are more general, but consequently less efficient than the Whittle Index approach in settings where indexability assumption holds. 

The S-RMAB problem has been studied in a more restricted setting by \cite{zayas2019asymptotically}. %finite horizon problem with streaming arms have been studied by \cite{zayas2019asymptotically}. 
They assume that, at each time step, arms may randomly arrive and depart due to random abandonment. However, the main limitation of their solution is the assumption that all arms have the same state-transition dynamics. This assumption does not hold in most of the real-world instances and thus, in this paper, we consider heterogeneous arms---arms are allowed to have their own transition dynamics. We show empirically that our algorithms perform well even with heterogeneous arms. 

%Computing the optimal policy for RMABs has been shown to be PSPACE hard in general, even with known state transition dynamics \cite{}. Whittle proposed a heuristic, 

Another related category of work studied \emph{sleeping arms} for the \emph{stochastic multi-armed bandits} (SMAB) problem, where the arms are allowed to be absent at any time step~\cite{kanade2009sleeping,kleinberg2010regret,biswas2015truthful}. However, the SMAB is different from RMAB because, in the former, when an arm is activated, a reward is drawn from a Bernoulli reward distribution (and not dependent on any state-transition process). Thus, the algorithms and analysis of SMABs do not translate to the RMAB setting.

%\paragraph{WCMDP approaches, Meleau et al. etc}

%\paragraph{Queuing Theory related}

\section{Streaming bandits}\label{sec:formulation}

%Classically, RMAB instances are assumed to have a fixed population~(\cite{whittle1988restless}), that is, it consists of a finite set of $N$ arms which remains in the system throughout the horizon. 
%Additionally, there is a budget constraint $k$ specifying the number of arms that can be pulled at any time step. 
The \emph{streaming restless multi-armed bandit} (S-RMAB) problem is a general class of RMAB problem where a stream of arms arrive over time (both for finite and infinite-horizon problems). Similar to RMAB, at each time step, the decision maker is allowed to take active actions on at most $k$ of the available arms. Each arm $i$ of the S-RMAB is a Partially Observable Markov Decision Process (POMDP)---represented by a $4$-tuple $(\mathcal{S}, \mathcal{A}, \mathcal{P}, r)$. $\mathcal{S}=\{0,1\}$ denotes the state space of the POMDP, representing the ``bad'' state (say, patient not adhering to the health program) and ``good'' state (patient adhering),  %\footnote{In the healthcare setting, the states represent a patient's adherence to the health program. For example, if a patient adheres to the health program, it is considered to be in the ``good'' state; and otherwise in the ``bad'' state.}
respectively. $\mathcal{A}$ is the action space, consisting of two actions $\mathcal{A}=\{a,p\}$ where an action $a$ (or, $p$), denotes the active (or, passive) action. The state $s \in \mathcal{S}$ of the arm, transitions according to a known transition function, $P_{s,s'}^{a, i}$ if the arm is pulled and according to the known function, $P_{s,s'}^{p, i}$ otherwise. We also assume the transition function to conform to two natural constraints often considered in existing literature \citep{zhao_liyu_paper, mate2020collapsing}: (i) Interventions should positively impact the likelihood of arms being in the good state, i.e. $P_{01}^a>P_{01}^p$ and $P_{11}^a>P_{11}^p$ and (ii) Arms are more likely to remain in the good state than to switch from the bad state to the good state, i.e. $P_{11}^a>P_{01}^a$ and $P_{11}^p>P_{01}^p$. Though the transitions probabilities are known to the planner, the actual state change is stochastic and is only partially observable---that is, when an arm is pulled, the planner discovers the true state of the arm; however, when the arm is not pulled, uncertainty about the true state persists. Under such uncertainties, it is customary to analyze the POMDP using its equivalent belief state MDP representation instead \citep{mate2020collapsing}. The state space of this MDP is defined by a set of all possible ``belief'' values that the arm can attain, denoted by $\mathcal{B}_i$. Each belief state $b \in \mathcal{B}_i$ represents the likelihood of the arm being in state $1$ (good state). This likelihood is completely determined by the number of days passed since that arm was last pulled and the last observed state of the arm \citep{zhao_liyu_paper}. %where the only information that affects our belief of an arm being in state 1 is the number of days since that arm was last pulled and the state ω observed at that time. %Equivalently, this can be expressed by the observation function $\mathcal{O}$, which is the identity matrix $I_{2 \times 2}$ when $a=1$ and a null matrix when $a=0$. 
At each time step $t$, the planner accrues a state-dependent reward $r_t$ from an active arm $i$, defined as: 

\[r_t(i)= 
\begin{cases}
 0~\text{if $s_t(i)=0$ (arm $i$ is in the bad state at time $t$)}\\
1 ~\text{if $s_t(i)=1$ (arm $i$ is in the good state at time $t$).}\\
\end{cases}
\]

The total reward\footnote{For a natural number $N$, we use the notation $[N]:= \{1,...N\}$.} of $R_t=\sum_{i\in [N]}(r_t(i))$ is accrued by the planner at time $t$, which is the sum of individual rewards obtained from the available arms. The planner's goal is to maximize her total reward collected, $\bar{R} := \sum_{t \in [T]}R_t$. 
This reward criterion is motivated by our applications in the healthcare intervention domain: interventions here correspond to reminding patients to adhere to their medication schedules and the good and bad states refer to patients either adhering or not adhering. The planner's goal is to maximize the expected number of times that all patients in the program adhere to their medication schedules. However, due to the limited budget, the planner is constrained to pull at most $k$ arms per time step. Assuming a set of $N$ arms, the problem then boils down to determining a policy, $\pi:\mathcal{B}_1\times\ldots\mathcal{B}_N \rightarrow \mathcal{A}^N$ which governs the action to choose on each arm given the belief states of arms, at each time step, maximizing the total reward accumulated across $T$ time steps. 
%To this end, \cite{whittle1988restless} proposed a Lagrangian relaxation of the RMAB problem which was shown to be asymptotically optimal under when a technical condition, called indexability, was satisfied. However, all these results consider the set of $N$ arms to arrive at the same time and stay forever (infinite horizon).

Contrary to previous approaches that typically consider arms to all arrive at the beginning of time and stay forever, in this paper we consider streaming multi-armed bandits---a setting in which arms are allowed to arrive asynchronously and have finite lifetimes. We denote the number of arms arriving and leaving the system at a time step $t\in[T]$ by $X(t)$ and $Y(t)$, respectively. Each arm $i$ arriving at time $t$, is associated with a fixed lifetime $L_i$ (for example, $L_i$ can be used to represent the duration of the health program for a patient, which is known to the planner). The arm consequently leaves the system at time $t + L_i$. Thus, instead of assuming a finite set of $N$ arms throughout the entire time horizon, we assume that the number of arms at any time $t$ is denoted by the natural number $N(t)$, and can be computed as $N(t)= \sum_{s=1}^{t}(X(s)-Y(s))$.  Thus, the goal of the planner is to decide, at each time step $t$, which $k$ arms to pull (out of the $N(t)\gg k$ arms, relabeled as [N(t)] each timestep for ease of representation), in order to maximize her total reward,
\begin{equation}\bar{R}:= \displaystyle\sum_{t\in[T]}\ \sum_{i\in [N(t)]}r_t(i).\end{equation}

\section{Methodology}

The dominant paradigm for solving RMAB problems is the Whittle index approach. The central idea of the Whittle approach is to decouple the RMAB arms and then compute indices for each arm that capture the ``value'' of pulling that arm. The Whittle Index policy then proceeds by pulling the $k$ arms with the largest values of Whittle Index. This greedy approach makes the time complexity linear in the number of arms, as indices can be computed independently for each arm. The computation of the index hinges on the notion of a ``passive subsidy'' $m$, which is the amount rewarded to the planner for each arm kept passive, in addition to the usual reward collected from the arm. The Whittle Index for an arm is defined as the infimum value of subsidy, $m$ that must be offered to the planner, so that the planner is indifferent between pulling and not pulling the arm.
To formalize this notion, consider an arm of the bandit in a belief state $b$. Its active and passive value functions, under a discount factor of $\beta$, and when operating under a passive subsidy $m$, can be written as:
\begin{align}
    V_{m, T}^{p}(b) &= b + m+ \beta  V_{m, T-1}(bP_{11}^p + (1-b)P_{01}^p) \label{eq:passive} \\
    V_{m, T}^{a}(b) &= b + \beta  b V_{m, T-1}(P_{11}^a) + \beta (1-b)V_{m, T-1}(P_{01}^a) \label{eq:active}
\end{align}
The value function for the belief state $b$ is $V_{m, T}(b) = \max \{V_{m, T}^{p}(b), \\ V_{m, T}^{a}(b) \}$. The Whittle Index for the belief state $b$, with a residual lifetime $T$ is defined as: $\inf_{m}\{m:V_{m, T}^{p}(b)=V_{m, T}^{a}(b)\}$.
The Whittle Index approach is guaranteed to be asymptotically optimal when a technical condition called \textit{indexability} holds for all the arms. Intuitively, indexability requires that if for some passive subsidy $m$, the optimal action on an arm is passive, then $\forall m' >m$, the optimal action should still remain passive. Equivalently, indexability can be expressed as: 
$\frac{\partial}{\partial m}  V_{m, T}^p(b) \ge \frac{\partial }{\partial m} V_{m, T}^a(b)$.

In this section we first show theoretically that the Streaming Bandit setup is indexable (subsection~\ref{sec:indexability}). Next, in subsection~\ref{sec:index-decay}, we observe and formalize a useful phenomenon about the Whittle Index in the finite horizon setting. We use this phenomenon to design fast algorithms for S-RMABs in subsection~\ref{sec:algo} and we provide runtime complexity analysis for the same in subsection~\ref{sec:complexity}. Finally in subsection~\ref{sec:reverse-threshold} we identify cases beyond those identified by previous work to which our efficient algorithm extends. 

\subsection{Conditions for indexability of streaming bandits}
\label{sec:indexability}
In this section, we extend the conditions for indexability that \cite{mate2020collapsing} originally established for infinite horizon, to the finite horizon setting of Streaming bandits. To show indexability, we first show in Theorem~\ref{thm:srmab-reduction-to-rmab}, that S-RMABs can be reduced to a standard RMAB with augmented belief states. We build on this result and prove another useful Lemma, both of which combined can be used to show that indexability holds for this augmented RMAB instance, and ultimately for S-RMABs (Theorem~\ref{thm:indexable}).   
\begin{definition}[Threshold Optimality~\cite{mate2020collapsing}]
An RMAB instance is called \textit{threshold optimal} if either a \textit{forward threshold} policy or a \textit{reverse threshold} policy is optimal. A forward (or reverse) threshold policy $\pi$ is optimal if there exists a threshold $b^*$ such that it is optimal to take a passive (or active) action whenever the current belief of the arm is greater than  $b^*$, that is, $\pi(b)=0$ (or $\pi(b)=1$) whenever $b>b^*$ and  $\pi(b)=1$ (or $\pi(b)=0$) whenever $b\leq b^*$. 
\end{definition}

First, we show that the belief state MDP of a Streaming Bandit arm with deterministic arrival and departure time can be formulated as an augmented belief state MDP of the same instance with infinite horizon. Using this, we prove that, whenever the infinite horizon problem satisfies threshold optimality for a passive subsidy $m$, then the augmented belief state MDP for finite horizon also satisfies threshold optimality. Using the result that indexability holds whenever threshold optimality is satisfied \cite{mate2020collapsing}, we imply that the Streaming Bandits problem is indexable whenever threshold optimality on the underlying infinite horizon problem is satisfied. 

\begin{theorem}
\label{thm:srmab-reduction-to-rmab}
The belief state transition model for a $2$-state Streaming Bandit arm with deterministic arrival time $T_1$ and departure time $T_2$ can be reduced to a belief state model for the standard restless bandit arm with  $T_2 + (T_2 - T_1)^2$ states.
\end{theorem}

\begin{proof}
Consider a streaming arm, that arrives (or, becomes available to the system) at time step $T_1$ and exits (or, becomes unavailable) at time step $T_2$. To capture the arm's arrival and departure in the belief model, we construct a new belief model with each state represented by a tuple $\langle$~\texttt{behavior, time-step}~$\rangle$, where  \texttt{behavior} takes a belief value in the interval $(0,1)$ or is set to $U$ (unavailable). $U$ can be set to any constant value (such as $U=0$). The transition probabilities are constructed as follows:
 \begin{itemize}
     \item The first $T_1-1$ states represent the unavailability of the arm and have deterministic transitions, i.e., for an action $a$,\\
     $P_{\langle U, t-1\rangle, \langle U, t\rangle}^{a} = 1$ for all $t\in\{2,\ldots,T_1-1\}$.
     \item At time $T_1$, the arm can either be in good state or bad state, so we create two states $\langle 1, T_1\rangle$ and $\langle 0, T_1\rangle$. For each $x\in\{0,1\}$, $P_{\langle U, T_1-1\rangle, \langle x, T_1\rangle}^{a} = p_x$ where $p_x$ represents the probability that the arm starts at a good ($1$) or bad ($0$) state. Note that, in our experiments, we assume that the initial state of an arm is fixed to $0$ or $1$, and can be captured by using either $p_x=0$ or $p_x=1$, respectively.
     \item For each time step $t\in\{T_1+1, T_2-1\}$, we create $2(t-T_1+1)$ states: $\langle b_w(0), t\rangle, \ldots, \langle b_w(t-T_1), t\rangle$ for each action $w\in\{0,1\}$. For any $t', t''\in\{0,1,\ldots, t-T_1\}$, the probability of transitioning from the state $\langle b_w(t'), t-1\rangle$ to the state $\langle b_w(t''), t+1\rangle$ is same as the probability of changing from belief value $b_w(t')$ to $b_w(t'')$ in one time step on taking action $w$. 
     \item For time step $t\geq T_2$, we create one sink state $\langle U, T_2\rangle$. This state represents unavailability of the arm subsequent to time step $T_2-1$. For any $t'\in\{0,1,\ldots,T_2-T_1\}$, the probability of transitioning from $\langle b_w(t'), T_2\rangle$ to $\langle U, T_2\rangle$ is $1$. 
 \end{itemize} 

 Thus, the number of states in the new belief network is:
 \begin{eqnarray}
 && T_1 - 1 + 2(1+\ldots+(T_2-T_1)) + 1\\ \nonumber
   &=& T_1 + (T_2-T_1)(T_2-T_1+1) \\      \nonumber
 &=& T_2 + (T_2 - T_1)^2
 \end{eqnarray}
 Thus, $T_2 + (T_2 - T_1)^2$ states are required for converting a belief network representing $2$-state streaming bandits problem to a classic RMAB problem.  
%  $\null\hfill\square$
 \end{proof}
 
 \begin{lemma}
If a forward (or reverse) threshold policy $\pi$ is optimal for a subsidy $m$ for the belief states MDP of the infinite horizon problem, then $\pi$ is also optimal for the augmented belief state MDP.
\end{lemma}

\begin{proof}
First, we define the value function for the modified belief states.
\[V_{m}^{p}(\langle b, t \rangle) =
\left\{
	\begin{array}{ll}
		b + m+ \beta  V_{m}(\langle bP_{11}^p + (1-b)P_{01}^p , t+1 \rangle)  & \mbox{if } b\neq U \\
		b + m + V_m(\langle b', t+1\rangle) & \mbox{otherwise }
	\end{array}
\right.\]
\[V_{m}^{a}(\langle b, t \rangle) =
\left\{
	\begin{array}{ll}
		b + \beta  (V_{m}(\langle bP_{11}^a, t+1 \rangle) + (1-b)V_m(\langle P_{01}^a, t+1 \rangle)) \\  
		\quad\quad\quad\quad\quad\quad\quad\quad\quad\quad\quad\quad \mbox{if } b\neq U \\
		b + V_m(\langle b', t+1\rangle) \quad \quad \quad \quad \mbox{otherwise }
	\end{array}
\right.\]
where $b'$ is the next belief state. 

The minimum value of $m_U$ that makes the passive action as valuable as active action at the states $\langle U, t \rangle$ for $T_1\leq t < T_2$, can be obtained by equating
\begin{eqnarray}
V_{m_U}^p(\langle U,t\rangle) &=& V_{m_U}^a(\langle U,t\rangle)\\
\Rightarrow U+m_U+V_{m_U}(\langle b', t+1\rangle) &=& U+V_{m_U}(\langle b', t+1\rangle)\\
\Rightarrow m_U &=& 0.
\end{eqnarray}
Assuming that there exists a forward (or reverse) threshold policy, $m_U=0$ implies that, even without any subsidy, passive action is as valuable as active action. 
%To show that the passive action is optimal at the $U$ states, 

Further, we show in the Appendix that the minimum subsidy at any other belief state is greater than $0$. As the belief states $b \neq U$ require a positive subsidy for the passive action to be optimal, while for the belief state $U$, passive is already optimal for a subsidy of zero, a policy that maximizes value while paying minimum subsidy, would never choose to set arms currently in the $u$ state to active. 

% $\hfill \square$ 
\end{proof}

\begin{theorem}\label{thm:indexable}
A Streaming Bandits instance is indexable when there exists an optimal policy, for each arm and every value of $m\in \mathbb{R}$, that is forward (or reverse) threshold optimal policy.
\end{theorem}
\begin{proof}
Using Theorem~$1$ and Lemma~$1$, it is straightforward to see that an optimal threshold policy for infinite horizon problem can be translated to a threshold policy for Streaming bandits instance. Moreover, using the fact that the existence of threshold policies for each subsidy $m$ and each arm $i\in N$ is sufficient for indexability to hold (Theorem $1$ of \cite{mate2020collapsing}), we show that the Streaming bandit problem is also indexable.  
% $\hfill\square$
 \end{proof}

\subsection{Index decay for finite horizons}
\label{sec:index-decay}
%Because the Threshold Whittle algorithm computes the exact Whittle Index for bandits with infinitely lived arms, we consider this as the limiting case of arms with finite horizon as the horizon diverges. For short horizons we observe a phenomenon that we call \emph{index decay}: index values decrease as the residual lifetime of an arm approaches zero. We formalize this observation in Theorem \ref{thm:index_decay}. We proceed by stating one fact and proving one useful Lemma, building up towards the Theorem. 

In this section we describe a phenomenon called \textit{index decay} which is observed considering short horizon. Here, the Whittle index values are low when the residual lifetime of an arm is $0$ or $1$. We formalize this observation in Theorem \ref{thm:index_decay}. We use this phenomenon as an anchor to develop our algorithm~(detailed in \ref{sec:algo}). We proceed by stating one fact and proving one useful Lemma, building up towards the Theorem.

\begin{fact}
\label{fact:linearFunctions}
For two linear functions, $f(x)$ and $g(x)$ of $x$, such that $f'(x)\ge g'(x)$, whenever $f(x_1)<g(x_1)$ and $f(x_2)=g(x_2)$, the following holds: $x_2>x_1$.
\end{fact}

\begin{lemma}
\label{lemma:active-rho-greater}
Consider an arm operating under a passive subsidy $m$. Assuming an initial belief state $b_0$, let $\rho^a(b_0, t)$ and $\rho^p(b_0, t)$ denote the probability of the arm being in the good state at time $t~\forall t<T$ when policies $\pi^a(t)$ and $\pi^p(t)$ are adopted respectively, such that  $\pi^a(0)=a$, $\pi^p(0)=p$, and $\pi^a(t)=\pi^p(t)~\forall t\in\{1,\ldots,T\}$. Then, $\rho^a(b_0, t)>\rho^p(b_0, t)~\forall t\in\{1,\ldots,T\}.$
\end{lemma}
% \textit{Proof.} We prove this using induction. For the base case, clearly, for $t=1$:
% \begin{align}
%     \rho^a(b_0, 1) &= b_0P_{11}^a + (1-b_0)P_{01}^a > b_0P_{11}^p + (1-b_0)P_{01}^p\quad\quad  (\because  P_{s1}^a>P_{s1}^p)\\
%         &= \rho^p(b_0, 1)
% \end{align}
% For the inductive step, we assume $\rho^a(b_0, t)>\rho^p(b_0, t)$ and show that it implies $\rho^a(b_0, t+1)>\rho^p(b_0, t+1)$.\\ 
% \begin{align}
%     \rho^a(b_0, t+1) &= \rho^a(b_0, t) P_{11}^{\pi(t)} + (1-\rho^a(b_0, t))P_{01}^{\pi(t)} = \rho^a(b_0, t) ( P_{11}^{\pi(t)}- P_{01}^{\pi(t)}) +P_{01}^{\pi(t)} \\
%     &> \rho^p(b_0, t) ( P_{11}^{\pi(t)}- P_{01}^{\pi(t)}) +P_{01}^{\pi(t)}= \rho^p(b_0, t+1)
% \end{align}
% The inductive hypothesis follows, completing the proof. $\null\hfill\square$

\begin{theorem}[Index Decay]\label{thm:index_decay}

Let $V_{m, T}^{p}(b)$ and $V_{m, T}^a(b)$ be the $T$-step passive and active value functions for a belief state $b$ with passive subsidy $m$. Let $m_{T}$ be the value of subsidy $m$, that satisfies the equation $V_{m,T}^p(b)=V_{m,T}^a(b)$ (i.e. $m_{T}$ is the Whittle Index for a residual life time of $T$). Assuming indexability holds, we show that: $\forall T > 1\colon$ $m_{T} > m_{1} > m_0=0$.  
\end{theorem}

% \textit{Proof Sketch.} The key idea is to use Fact~\ref{fact:linearFunctions} with $f=V_{m, T}^p(b)$, $g=V_{m, T}^a(b)$, $x_1=m_1$ and $x_2=m_{T}$ to show that the Whittle index decays for short horizon. 

\begin{figure}
\centering
\includegraphics[width=0.5\textwidth]{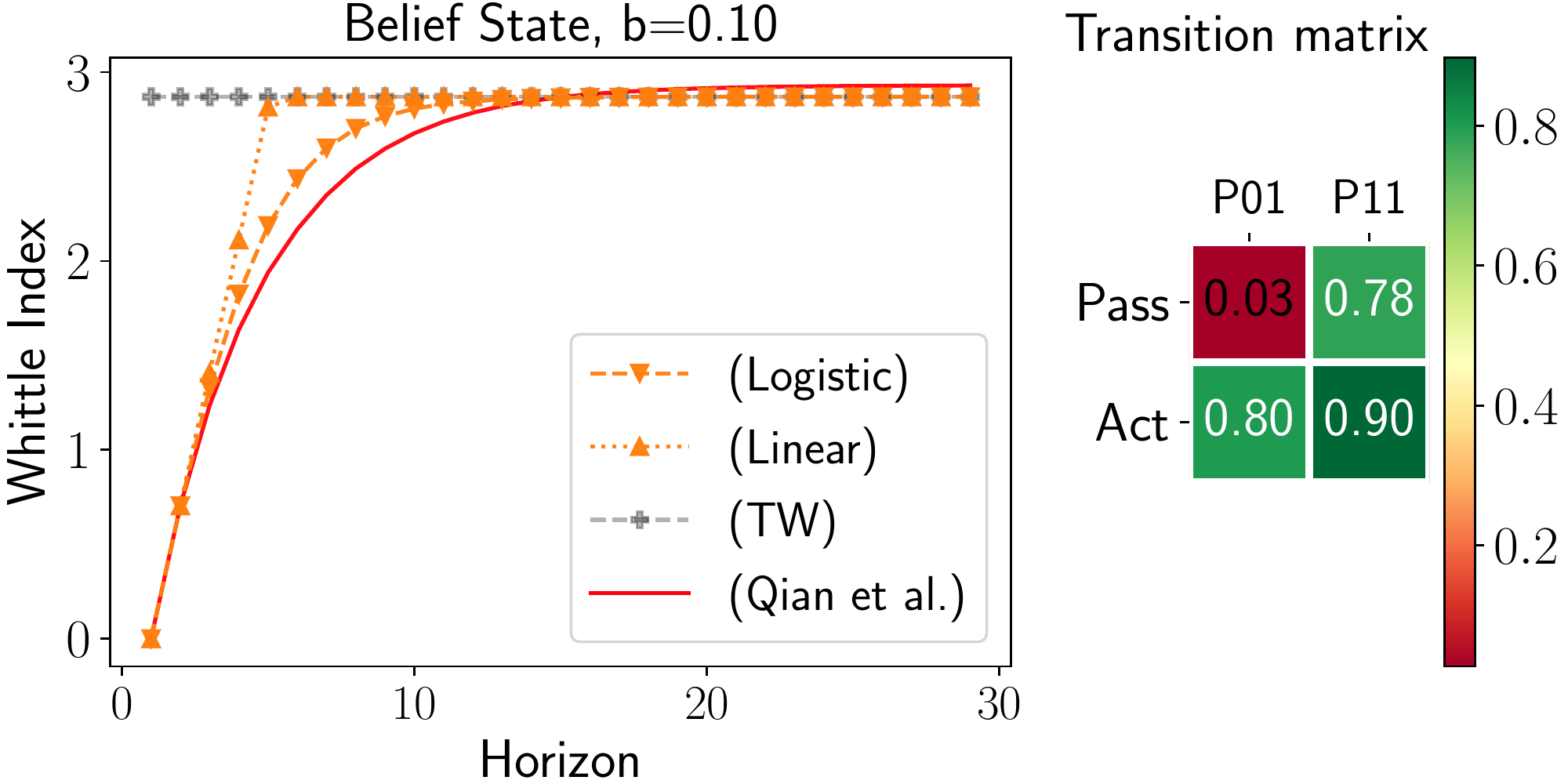}
\caption{Whittle Indices for a belief state as computed by different algorithms. Both our algorithms capture index decay providing good estimates.}%
\label{fig:approximation}
\end{figure}
% \begin{assumption}\label{assumption}
% The passive value function is more sensitive to changes in the subsidy than the active value function: $\frac{\partial}{\partial m}  V_{m, T}^p(b) >\frac{\partial }{\partial m} V_{m, T}^a(b)$. This assumption is common for RMABs and is implied by indexability for infinite horizons \citep{mate2020collapsing}.
% \end{assumption}
\begin{proof}
We provide our argument for a more general reward criterion than the total reward introduced in Section \ref{sec:formulation}. Consider a discounted reward criterion with discount factor $\beta\in[0,1]$ (where $\beta=1$ corresponds to total reward). 
$m_0$ is simply the $m$ that satisfies: $V_{m, 0}^{p}(b)=V_{m, 0}^{a}(b)$ i.e., $b+m =b $, thus $m_0=0$. 
Similarly, $m_1$ can be solved by equating $ V_{m_1, 1}^{p}(b)$ and $V_{m_1, 1}^{a}(b)$ and obtained as: $m_1= \beta \Delta b = \beta \Big(\big(b~P_{11}^{a} + (1-b)~P_{01}^{a}\big) - \big(b~P_{11}^{p} + (1-b)~P_{01}^{p}\big)\Big)$

% \begin{align}
%     %& V_{m_1, 1}^{p}(b)=V_{m_1, 1}^{a}(b) \\
%     \implies & b + m_1+ \beta \big(bP_{11}^p + (1-b)P_{01}^p\big) = b + \beta \big(b (P_{11}^a) + (1-b)(P_{01}^a)\big) \nonumber\\
%     \implies & m_1 = \beta \big(b (P_{11}^a - P_{11}^p ) + (1-b)(P_{01}^a-P_{01}^p)\big)\label{eq:m1}
% \end{align}
Using the natural constraints $P_{s1}^a > P_{s1}^p$ for $s \in \{0,1\}$, we obtain $m_1>0$.

Now, to show $m_{T}>m_1~\forall T > 1$, we first show that $V_{m_1, T}^a(b)>V_{m_1, T}^p(b)$. Combining this with the fact that $V_m(.)$ is a linear function of $m$ and by definition, $m_{T}$ is a point that satisfies $V_{m_{T}, T}^p(b)=V_{m_{T}, T}^a(b)$, we use Fact~\ref{fact:linearFunctions} and set $f=V_{m, T}^p(b)$, $g=V_{m, T}^a(b)$, $x_1=m_1$ and $x_2=m_{T}$ to obtain $m_{1}<m_{T}$, and the claim follows. To complete the proof we now show that $V_{m_1, T}^a(b)>V_{m_1, T}^p(b)$. 

Starting from an initial belief state $b_0$, let $\rho^p(b_0, t)$ be the expected belief for the arm at time $t$, if the passive action was chosen at $t=0$ and the optimal policy, $\pi^p(t)$ was adopted for $0<t<T$. Similarly let $\rho^a(b_0, t)$ be the expected belief at time $t$, if the active action was chosen at $t=0$ and the \emph{same} policy, $\pi^p(t)$ (which may not be optimal now) was adopted for $0<t<T$. Then, $\beta \big(\rho^a(b_0, 1)- \rho^p(b_0, 1) \big) =m_1>0$ as shown above. 
%$(b_{t}^{\pi(t-1)}|b_{t-1})$ be the expected belief for an arm at time $t$, if $b_{t-1}$ was the expected belief of the arm at time $t-1$ and an action $\pi(t-1)$ was taken. Then, $(b_1^a|b_0) - (b_1^p|b_0) =m_1 >0$ as shown above. Now let $\pi^p(t)$ be an optimal policy that selects actions for $t \in {1, \dots , T-1}$ when the passive action is taken on the $0^{th}$ time step.
Note that if we took actions according to $\pi^p(t)$ for $t \in \{1, \dots , T-1\}$ with active action taken at the $0^{th}$ time step, the total expected reward so obtained is upper bounded by the active action value function, $ V_{m_1, T}^a(b_0)$. Thus, 
\begin{align}
    V_{m_1, T}^p(b_0) = b_0 + m_1 & +  \beta\rho^p(b_0, 1) + \sum_{t=2}^{T}\beta^{t}\rho^p(b_0, t) \\ 
    &+ \big( \sum_{t=1}^{T}\beta^t m_1.\mathbb{1}_{\{\pi^p(t)=passive\}}\big) \nonumber \\
    = b_0 + \beta\rho^a(b_0, 1) +& \sum_{t=2}^{T}\beta^{t}\rho^p(b_0, t) +\big( \sum_{t=1}^{T} \beta^t m_1.\mathbb{1}_{\{\pi^p(t)=passive\}}\big) \nonumber \\
    < b_0 + \beta\rho^a(b_0, 1) +& \sum_{t=2}^{T}\beta^{t}\rho^a(b_0, t)  +\big( \sum_{t=1}^{T} \beta^t m_1.\mathbb{1}_{\{\pi^p(t)=passive\}}\big) \\ 
    & \quad\quad\quad\quad\quad\quad\quad \text{(by Lemma \ref{lemma:active-rho-greater})} \nonumber
    %&= b_0 + \rho^a(b_0, t)+ b_2^{\pi(1)}|(b_1^p|b_0) \sum_{t=3}^{T}b_t^p|b_{t-1}\\
    %&+\big( \sum_{t=1}^{T} m_1.\mathbb{1}_{\{\pi(t)=passive\}}\big) \\
    %&\le b_0 + \rho^a(b_0, t)+ b_2^{\pi(1)}|(b_1^a|b_0) \sum_{t=3}^{T}b_t^p|b_{t-1}\\
    %&+\big( \sum_{t=1}^{T} m_1.\mathbb{1}_{\{\pi(t)=passive\}}\big) \\
    \\ &\le V_{m_1, T}^a(b_0)  \nonumber
\end{align}
%$\null\hfill\square$
\end{proof}

\subsection{Proposed algorithms}\label{sec:algo}

The key insight driving the design of our solution is that, by accounting for the index decay phenomenon, we can bypass the need to solve the costly finite horizon problem. We make use of the fact that we can cheaply compute index values for arms with residual lifetime $0$ and $1$, where the index decay phenomenon occurs, and for infinite horizon bandits. Our proposed solution for computing indices for arbitrary residual lifetime is to use a suitable functional form to interpolate between those three observations. We propose an interpolation template, that can be used to obtain two such algorithms, one using a piece-wise linear function %(Algorithm \ref{alg:linear-approximation}) 
and the other using a logistic function% (Algorithm \ref{alg:logistic-approximation})
.

Recall that we establish in Theorem~\ref{thm:index_decay} that the Whittle Index for arms with a zero residual lifetime, is always zero.
%is the infimum subsidy such that the passive and active actions are both equally optimal. In the case of arms with a zero residual lifetime (when the current time-step is the last), there are no more time-steps left to accrue rewards under either action. Thus the active and passive value functions can be equal only when the passive subsidy —- which is only received under the passive action —-is $0$. The Whittle Indices for arms with zero residual lifetime are thus always zero. 
Similarly, indices for arms with residual lifetime of $1$ are simply the myopic indices, computed as: 
\[\Delta b = \big(b~P_{11}^{a} + (1-b)~P_{01}^{a}\big) - \big(b~P_{11}^{p} + (1-b)~P_{01}^{p}\big).\]

% \begin{algorithm}[t]

%   \caption{Logistic Interpolation Algorithm}
%   \label{alg:logistic-approximation}
%   \begin{algorithmic}[1] 
%   \STATE Pre-compute the Threshold Whittle Oracle \texttt{TW($b, P^i$)} $~\forall b \in \mathcal{B}_i ~\forall~\text{arms}~i \in [N]$, with transition matrix, $P^i$ and set of belief states $\mathcal{B}_i$.
%   \STATE \textbf{Input:} $\Bar{b}_{N \times 1} \in [0,1]^N, ~\Bar{h}_{N \times 1} \in [L]^N$, containing the belief values and remaining horizons for the $N$ arms.
%   \STATE Initialize $\hat{W}_{N \times 1}$ to store estimated Whittle Indices.
%   \FOR{each arm $i$ in bandit}
%   \STATE Let $b \coloneqq \Bar{b}_i, h \coloneqq \Bar{h}_i$ and let $P$ be $i$'s transition matrix.
%   \STATE Compute the myopic index $\Delta b$ as: \newline 
%   $\Delta b = \big(b~P_{11}^{a} + (1-b)~P_{01}^{a}\big) - \big(b~P_{11}^{p} + (1-b)~P_{01}^{p}\big)$.
%   \STATE \texttt{TW}$\leftarrow$ $\texttt{TW(b, P)}$ (Threshold Whittle Oracle). 
%   \STATE Compute $C_1 \coloneqq 2\texttt{TW}$.
%   \STATE Compute $C_2 \coloneqq -\log\big((\frac{\Delta b}{C_1} +\frac{1}{2})^{-1}-1\big)$.
%   \STATE Compute $C_3 \coloneqq -\texttt{TW}$.
%   \STATE Set $\hat{W}_i= \frac{C_1}{1+e^{C_2h}}+C_3$.
%   \ENDFOR
%   \STATE Pull the $k$ arms with the largest values of $\hat{W}$.
%   \end{algorithmic}
% \end{algorithm}

For the linear interpolation, we assume $\hat{W}(h)$, our estimated Whittle Index, to be a piece-wise-linear function of $h$ (with two pieces), capped at a maximum value of the Whittle Index for the infinite horizon problem, corresponding to $h=\infty$. We denote Whittle Index for infinite horizon as $\overline{W}$. The first piece of the piece-wise-linear $\hat{W}(h)$ must pass through the origin, given that the Whittle Index is $0$ when the residual lifetime is $0$. The slope is determined by $\hat{W}(h=1)$ which must equal the myopic index, given by $\Delta b$. The second piece is simply the horizontal line $y=\overline{W}$ that caps the function to its infinite horizon value. The linear interpolation index value is thus given by 

\begin{equation}\label{eq:interpolation_linear}
\hat{W}(h,\Delta b, \overline{W}) = \min\{h~\Delta b,  \overline{W}\}.
\end{equation}

The linear interpolation algorithm performs well and has very low run time, as we will demonstrate in the later sections. However, the linear interpolation can be improved by using a logistic interpolation instead. The logistic interpolation algorithm yields moderately higher rewards in many cases for a small additional compute time. For the logistic interpolation, we let 

\begin{equation}\label{eq:interpolation_logistic}
\hat{W}(h,\Delta b, \overline{W}) = \frac{C_1}{1+e^{-C_2h}} +C_3.
\end{equation}

\begin{algorithm}[!t]
  \caption{Interpolation Algorithm Template}
  \label{alg:linear-approximation}
  \begin{algorithmic}[1] 
    \STATE Pre-compute $\overline{W}(b, P^i)~\forall b \in \mathcal{B}_i, ~\forall~i \in [N]$, with transition matrix $P^i$ and set of belief states $\mathcal{B}_i$.
  \STATE \textbf{Input:} $\Bar{b}_{N \times 1} \in [0,1]^N, ~\Bar{h}_{N \times 1} \in [L]^N$, containing the belief values and remaining lifetimes for the $N$ arms.
  \STATE Initialize $\hat{W}_{N \times 1}$ to store estimated Whittle Indices.
  \FOR{each arm $i$ in N}
  \STATE Let $b \coloneqq \Bar{b}_i, h \coloneqq \Bar{h}_i$ and let $P$ be $i$'s transition matrix.
  \STATE Compute the myopic index $\Delta b$ as: \\
   $\Delta b = \big(b~P_{11}^{a} + (1-b)~P_{01}^{a}\big) - \big(b~P_{11}^{p} + (1-b)~P_{01}^{p}\big)$.
  %\STATE \texttt{TW}$\leftarrow$ $\texttt{TW(b, P)}$ (Threshold Whittle Oracle). 
  \STATE Set $\hat{W}_i(h,\Delta b, \overline{W})$ according to one of the interpolation functions \eqref{eq:interpolation_linear} or \eqref{eq:interpolation_logistic}.
  \ENDFOR
  \STATE Pull the $k$ arms with the largest values of $\hat{W}$.
  \end{algorithmic}
\end{algorithm}

We now apply the three constraints on the Whittle Index established earlier and solve for the three unknowns $\{C_1, C_2, C_3\}$ to arrive at the logistic interpolation model. For the residual lifetimes of $0$ and $1$, we have that $\hat{W}(0) =0$ and $\hat{W}(1) =\Delta b$. As the horizon becomes infinity, $\hat{W}(.)$ must converge to $\overline{W}$, giving the final constraint $\hat{W}(\infty)= \overline{W}$. Solving this system yields the solution: \[C_1=2\overline{W},\ 
C_2=-\log\left(\left(\frac{\Delta b}{C_1}+ \frac{1}{2}\right)^{-1}-1\right),\  C_3=-\overline{W}.\]

We note that both interpolations start from $\hat{W}=0$ for $h=0$ and saturate to $\hat{W}= \overline{W}$ as $h \rightarrow \infty$.

We compare the index values computed by our interpolation algorithms with the exact solution by \cite{qian2016restless}. Figure \ref{fig:approximation} shows an illustrative example, plotting the index values as a function of the residual lifetime and shows that the interpolated values agree well with the exact values. 

\textbf{Infinite horizon index:} For transition matrices that satisfy the conditions for forward threshold policies to be optimal, \citet{mate2020collapsing} present an algorithm that computes $\overline{W}$ cheaply. The cornerstone of their technique is to leverage forward threshold optimality to map the passive and active actions to two different forward threshold policies, and find the value of subsidy $m$ that makes the expected reward of the policies equal. We extend this reasoning to reverse threshold optimal arms.  %for belief states that are less than the stationary belief value of the arm (i.e. $b< b_{stat} =\frac{P_{01}^{p}}{P_{01}^{p}+P_{10}^{p}}$) and defer the complete reasoning to the Appendix. 
% (detailed in Appendix).

%For such belief states, the belief value increases upon taking the passive action.

% We start by comparing the index values computed by our interpolation algorithms with the solution by \cite{qian2016restless} that computes Whittle Indices for infinite horizon only. Figure \ref{fig:approximation} shows an illustrative example, plotting the index values as a function of the residual lifetime. As can be seen, there is considerable agreement between the interpolated index values and the indices computed by the algorithm of \cite{qian2016restless}. We note that for transition matrices that satisfy the conditions for forward threshold policies to be optimal, the limits of the interpolation values and of \cite{qian2016restless} are guaranteed to agree (converging to the $\overline{W}$). 
%The linear interpolation algorithm follows a similar idea, replacing the logistic function with a piece-wise linear function, growing at rate $\Delta b$ until it reaches the saturation value $\texttt{TW}$ and remaining flat thereon. 
\subsection{Complexity analysis}
\label{sec:complexity}
%\hl{We also observe significant index decay of over one third of \texttt{TW} index values.}%We also observe significant index decay for a residual lifetime of $1$.

% \begin{figure}
%     \centering
%     \includegraphics[width=0.95\linewidth]{figs/fin_whittle.pdf}
%     \caption{Whittle Indices for two belief values as computed by different algorithms. Both our algorithms are able to capture index decay and provide good estimates.}
%     \label{fig:approximation}
% \end{figure}

%We now turn to the analysis of the computational complexity of our algorithm. \hl{TODO:}Denoting by ${N}$ the expected number of arms arriving each time step and $L$ their lifetimes, the solution of our algorithms (both versions) requires ${N}*|\mathcal{B}_i|=\bar{N}*2L\le\bar{N}*2T$ pre-computations, plus $\sum_{t \in [T]} N(t) * T \le \bar{N}*T$ calculations for the solution. The overall complexity of our algorithm is thus $\mathcal{O}(\bar{N}T)$.
For the complexity analysis of the algorithms, we denote by $\bar{X}$ the expected number of arms arriving each time step and $\bar{L}$ their average expected lifetimes. The expected number of arms at any point in time is then $\mathcal{O}(\bar{X}\bar{L})$ \citep{little1961proof}. Our algorithms (both versions) require a per-period cost of $\mathcal{O}(\bar{X}*|\mathcal{B}_i|=\bar{X}*2\bar{L})$ for the Threshold Whittle pre-computations, plus $\mathcal{O}(\bar{X})$ computations for the myopic cost, plus $\mathcal{O}(\bar{X}\bar{L}*\bar{L})$ calculations (for $\bar{X}\bar{L}$ arms, each requiring up to $\bar{L}$ additions or multiplications) and $\mathcal{O}(\bar{X}\bar{L})$ for determining the top $k$ indices. The overall per-period complexity of our algorithm is thus $\mathcal{O}(\bar{X}\bar{L}^2)$. 

For comparison, 
% Threshold Whittle, has per-period time complexity of $\mathcal{O}(\bar{X}\bar{L})$. The interpolation algorithm thus increases the complexity by a factor of $\bar{L}$. This is an acceptable trade-off when $\bar{L}$ is low, as is typically the case in our applications, since the difference in performance is significant. 
\citeauthor{qian2016restless} has a per-period complexity of $\approx\mathcal{O}\big( \bar{X}\bar{L}^{(3 + \frac{1}{18})}\log(\frac{1}{\epsilon})\big)$, where $\log(\frac{1}{\epsilon})$ is due to a bifurcation method for approximating the Whittle index to within error $\epsilon$ on each arm and $\bar{L}^{2 + \frac{1}{18}}$ is due to the best-known complexity of solving a linear program with $\bar{L}$ variables \citep{jiang2020faster}
. %\cite{qian2016restless} \hl{instead require XXX calculations}

\subsection{Reverse Threshold Arms} \label{sec:reverse-threshold}

Computing the infinite horizon Whittle index cheaply ($\overline{W}$) is key to the runtime efficiency of our approach. 
Existing methods provide techniques to compute $\overline{W}$ used in the previous subsection, when the transition matrices satisfy the forward threshold optimality conditions. In this subsection, we describe how the technique can be extended to the case when reverse threshold optimality conditions are satisfied.  

All the belief states that an arm can ever visit during its lifetime $L$ can be enumerated and organized into two chains — each chain corresponding to one of the two possible observations ($\omega \in \{0,1\}$) last observed for that arm. These chains are shown in Figure~\ref{fig:belief-chains}. \cite{mate2020collapsing} present an  algorithm to compute the index for forward threshold arms with belief states belonging to the NIB process (i.e. whenever $b>b_{stationary}=\frac{P_{01}^p}{P_{01}^p+P_{10}^p}$). The algorithm relies on mapping the active and passive actions to two different forward threshold policies (with corresponding threshold states on the two chains indexed as $X_0, X_1$) and equating the policies' rewards to solve for the passive subsidy $m$, that makes the two actions equal. 

We extend this reasoning to reverse threshold arms with belief chains belonging to the $\omega=0$ chain of the SB (split-belief) process, as shown in Figure~\ref{fig:belief-chains}. The belief states belonging to the increasing chain ($\omega=0$ chain) satisfy $b<b_{stationary}=\frac{P_{01}^p}{P_{01}^p+P_{10}^p}$. We identify two different reverse threshold policies that correspond to the active and passive actions, which can be used to set up similar indifference equations. For a given belief state on the increasing chain with index in the chain $X$, the corresponding reverse threshold policies can be indexed by $(X_0, X_1)= (1, X)$ and $(X_0, X_1)=(1, X+1)$ and used to solve for the whittle index using the indifference equation outlined in Algorithm 1 of \cite{mate2020collapsing}.  

\begin{figure}[h!]
    \centering
    \includegraphics[width=0.6\linewidth]{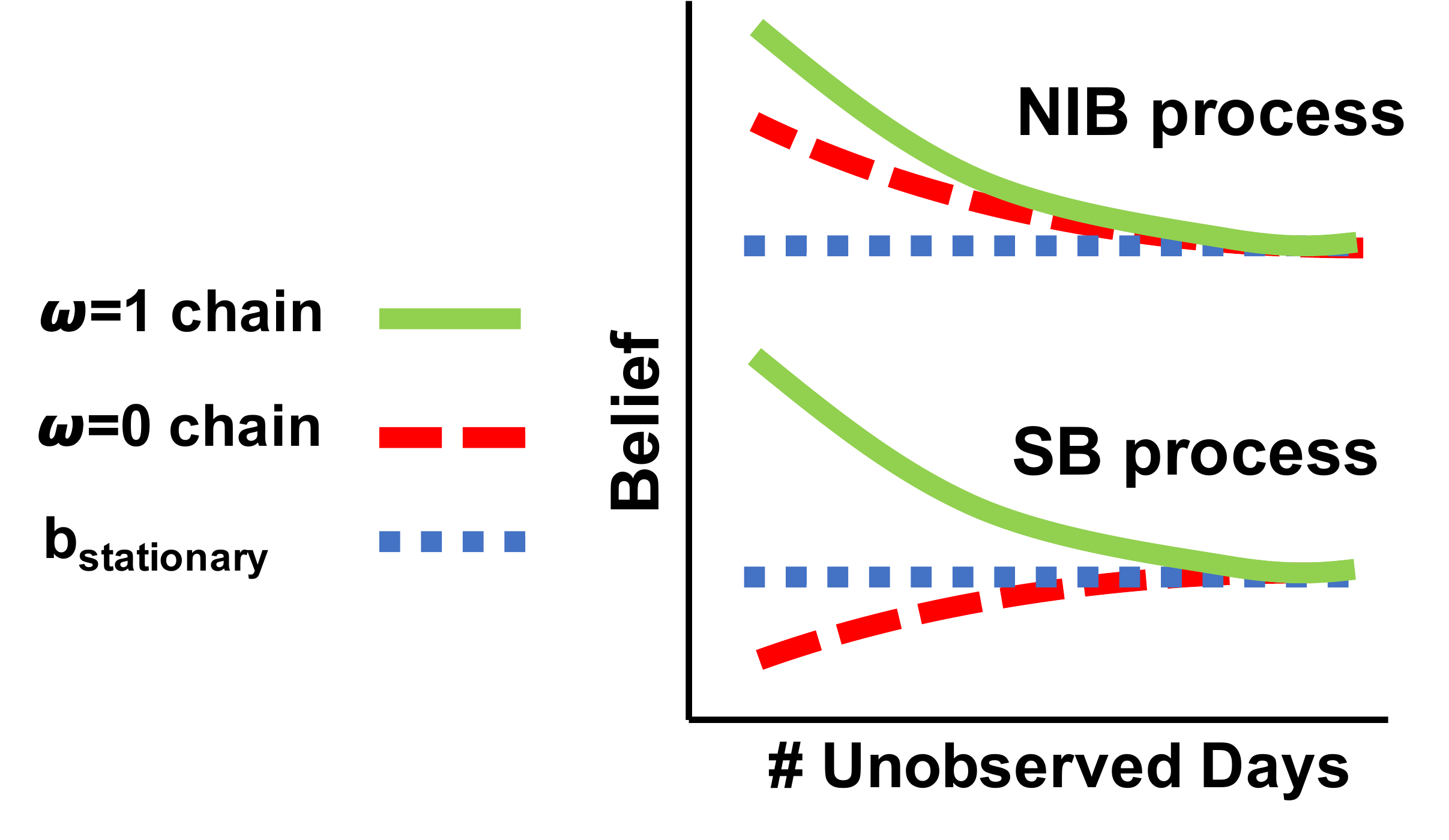}
    \caption{Belief values arranged in chains as presented in \cite{mate2020collapsing}. For every possible last observed state of the arm, $\omega$, there is a corresponding chain of belief states.}
    \label{fig:belief-chains}
\end{figure}

\section{Experimental evaluation}\label{sec:experiments}

We evaluate the performance and runtime of our proposed algorithms against several baselines, using both, real as well as synthetic data distributions. \textsc{Logistic} and \textsc{Linear} are our proposed algorithms. Our main baselines are: (1) a precise, but slow algorithm by \textsc{\citeauthor{qian2016restless}}, which accounts for the residual lifetime by solving the expensive finite-horizon POMDP on each of the $N$ arms and finds the $k$ best arms to pull and (2) Threshold-Whittle \citep{mate2020collapsing} (marked as \textsc{TW}), a much faster algorithm, that is only designed to work for infinitely long residual time horizons. \textsc{Myopic} policy is a popularly used baseline \citep{mate2020collapsing,qian2016restless, zhao_liyu_paper} that plans interventions optimizing for the expected reward of the immediate next time step. 
%The myopic strategy always uses the myopic index as a Whittle index.
\textsc{Random} is a naive baseline that pulls $k$ arms at random. 
% The different experimental settings are described in the following sections.

Performance is measured as the excess average intervention benefit over a `do-nothing' policy, measuring the sum of rewards over all arms and all timesteps minus the reward of a policy that never pulls any arms. Intervention benefit is normalized to set \cite{qian2016restless} equal to $100\%$ and can be obtained for an algorithm $\texttt{ALG}$ as: $\frac{100\times (\overline{R}^{\texttt{ALG}}- \overline{R}^{\text{No\  intervention}})}{\overline{R}^{\text{Qian et al.}}- \overline{R}^{\text{No\  intervention}}}$ where $\overline{R}$ is the average reward.  
%In the following subsections, we demonstrate the dramatic speed-up achieved by our algorithms over Qian et al., without having to sacrifice on performance. We also present extensive analyses that affirm the strong performance of our algorithm across different conditions. 
All simulation results are measured and averaged over 50 independent trials and error bars denote the standard errors. 
%The error bars show \hl{bootstrapped $95\%$ confidence intervals}. 
\begin{figure}[h!]
\begin{subfigure}{0.33\textwidth}
    \centering
    \includegraphics[width=0.85\linewidth]{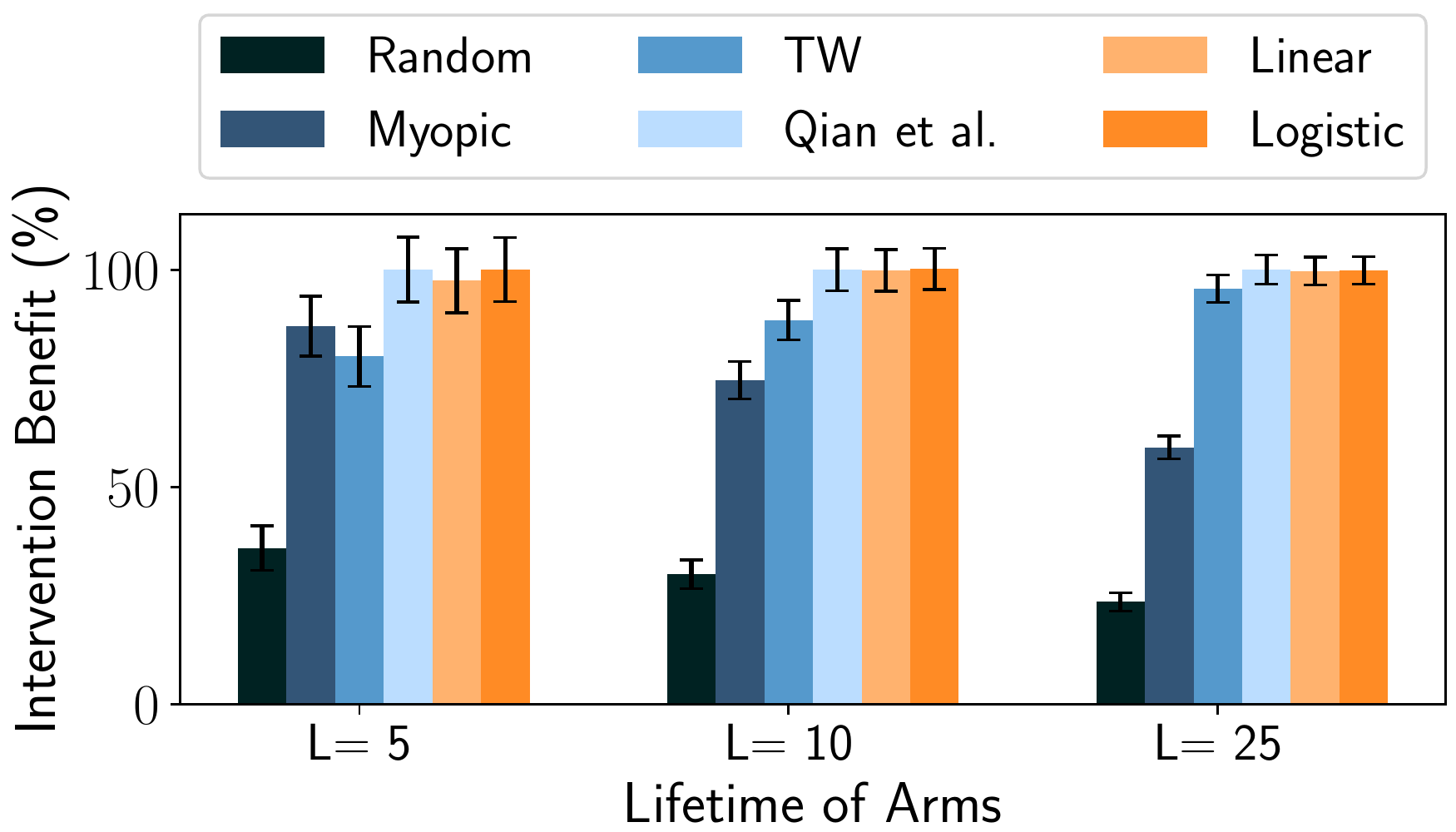}
    \caption{Arms arriving synchronously}
    \label{fig:ib-vs-L-SYNC}
\end{subfigure}
\begin{subfigure}{0.33\textwidth}
    \centering
    \includegraphics[width=0.85\linewidth]{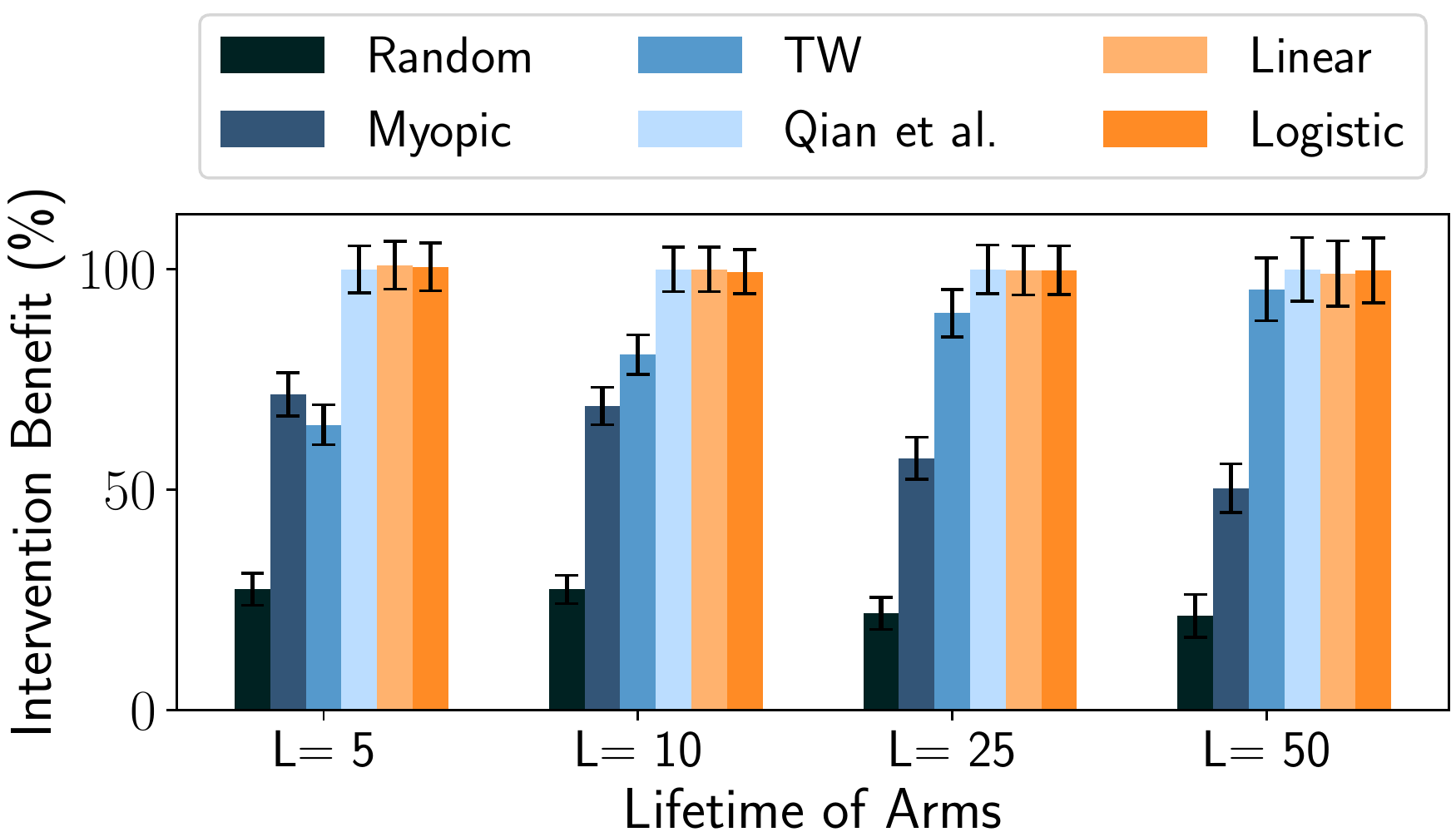}
    \caption{Arms arriving asynchronously as a stochastic process}
    \label{fig:ibPlot-vs-L-poisson}
\end{subfigure}
\caption{(a) Performance of Threshold Whittle algorithm degrades when the lifetime of arms gets shorter, even when all arms start synchronously (b) The performance dwindles further if arms arrive asynchronously.}
\end{figure}

\begin{figure*}[t]
\begin{subfigure}{0.2\textwidth}
    \centering
    \includegraphics[width=0.87\linewidth]{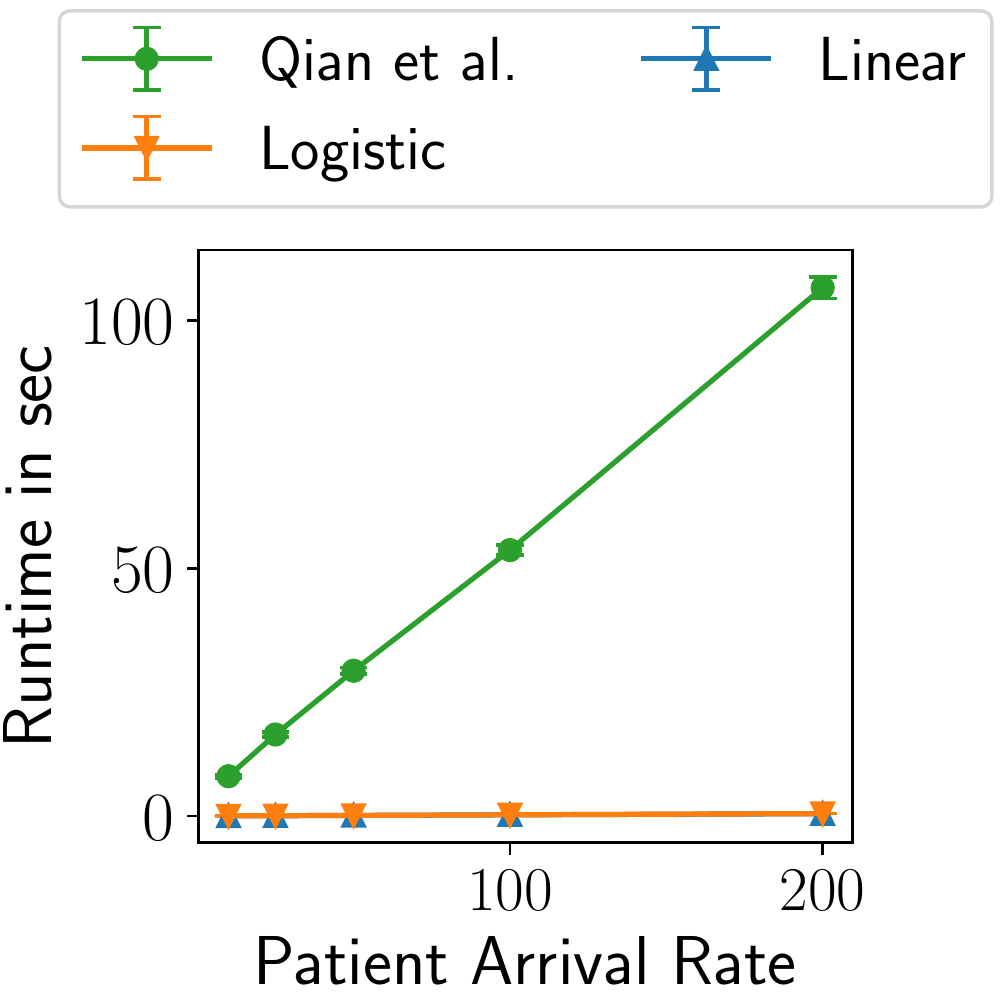}
    \caption{}
    \label{fig:runtime_tb}
\end{subfigure}
\begin{subfigure}{0.34\textwidth}
    \centering
    \includegraphics[width=0.87\linewidth]{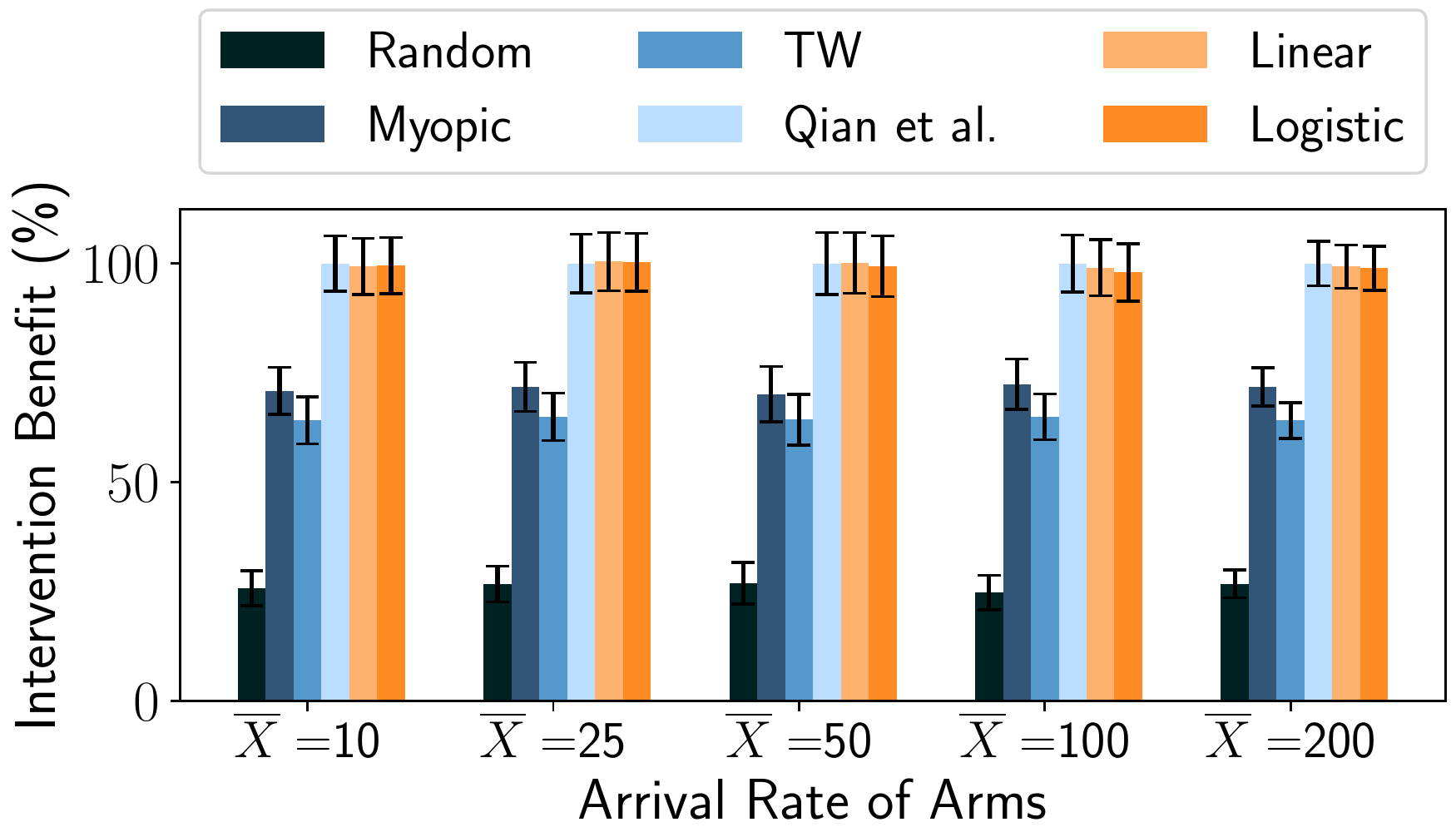}
    \caption{}
    \label{fig:ibPlot-vs-N-short}
\end{subfigure}
\begin{subfigure}{0.34\textwidth}
    \centering
    \includegraphics[width=0.88\linewidth]{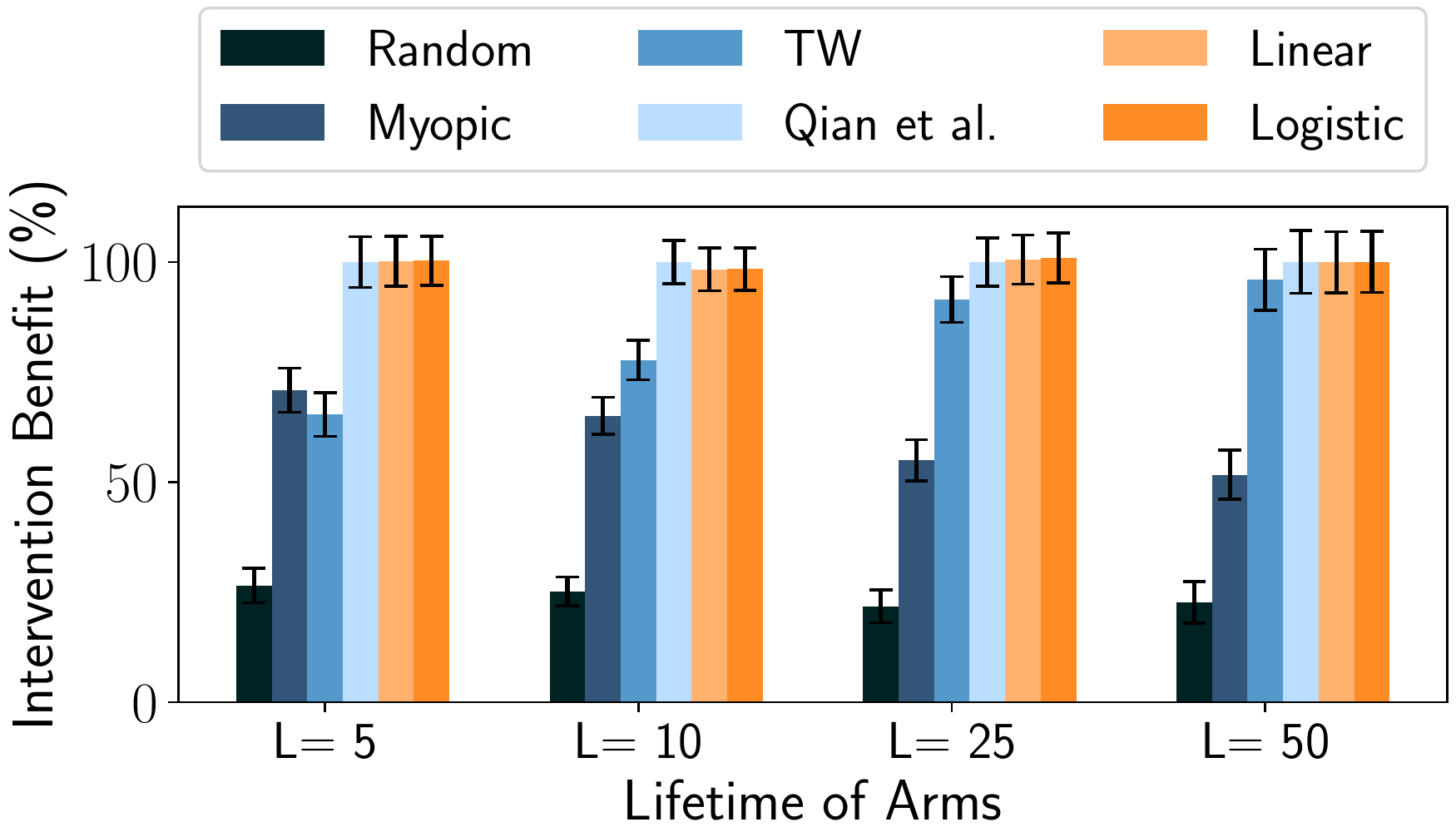}
    \caption{}
    \label{fig:ibPlot-vs-L-deterministic}
\end{subfigure}
\caption{(a) Linear and Logistic interpolation algorithms are nearly $200 \times$ faster than Qian et al. (b) \& (c) The interpolation algorithms achieve the speedup without sacrificing on performance, while other fast algorithms like Threshold Whittle deteriorate significantly for small residual horizons.}
\label{fig:arrival-rate-figure}
\end{figure*}

\subsection{Real domain: Monitoring tuberculosis medication adherence }\label{sec:TB_domain}
We first test on an anonymized real-world data set used by \cite{killian2019learning}, consisting of daily adherence data of tuberculosis patients in Mumbai, India following a prescribed treatment regimen for six months. For our study, we only obtain the summary statistics capturing the transition probabilities of these patients moving between the adherent and non-adherent states as extracted from the dataset. We then follow the same data imputation steps adopted by \cite{mate2020collapsing} for arriving at the transition matrices, $P_{ss'}^a$ and $P_{ss'}^p$ for each patient. We sample transition matrices from this real-world patient distribution and run simulations over a simulation length much longer than the lifetimes of the patients in the simulation.

In Figure~\ref{fig:ib-vs-L-SYNC}, we first demonstrate the impact of a short horizon alone on the performance of various algorithms in a simple, non-streaming setting. In Figure ~\ref{fig:ibPlot-vs-L-poisson}, we contrast this with a similar comparison for the short horizon setting combined with a stochastic incoming stream of patients.

In Figure~\ref{fig:arrival-rate-figure}, we again consider the finite horizon setting with a deterministic incoming stream of patients. In Figure \ref{fig:runtime_tb}, we plot the runtimes of our algorithms and that of \citeauthor{qian2016restless}, as a function of the daily arrival rate, $\bar{X}$ of the incoming stream. Figure \ref{fig:ibPlot-vs-N-short} measures the intervention benefits of these algorithms for these values of $\bar{X}$. The lifetime of each arm, $L$ is fixed to $5$ and the number of resources, $k$ is set to $10\% \times(\bar{X}L)$. Each simulation was run for a total length $T$ such that $\bar{X}T=5000$, which is the total number of arms involved in the simulation. Runtime is measured as the time required to simulate $L$ days. The runtime of \citeauthor{qian2016restless} quickly far exceeds that of our algorithms. For the $\bar{X}=200$ case, a single trial of  \citeauthor{qian2016restless} takes 106.69 seconds to run on an average, while the proposed Linear and Logistic interpolation algorithms take 0.47 and 0.49 seconds respectively, while attaining virtually identical intervention benefit. Other competing fast algorithms like Threshold Whittle, which assume an infinite residual horizon, suffer a severe degradation in performance for such short residual horizons. Our algorithms thus manage to achieve a dramatic speed up over existing algorithms, without sacrificing on performance. 
% Figure \ref{fig:runtime_all}(right) offers a zoomed-in view of the fast algorithms, showing the runtime of the linear and logistic interpolation algorithms is virtually same as the runtime of the plain threshold Whittle algorithm — showing that all grow linearly in the number of arms with only marginal differences in growth rates. 

% \begin{figure}[h!]
%     \centering
%     \includegraphics[width=0.95\linewidth]{figs/runtime-vs-n-both.pdf}
%     \caption{Linear and Logistic interpolation algorithms are nearly $150 \times$ faster than Qian et al.(left). The distinction between the runtimes of TW, Linear and Logistic interpolation algorithms is shown in the zoomed in figure on the right.}
%     \label{fig:runtime_all}
% \end{figure}

% \begin{figure}[h!]
%     \centering
%     \includegraphics[width=0.95\linewidth]{figs/ib-vs-n-short.pdf}
%     \caption{Intervention benefit of Linear and Logistic approximation remains high for varying sizes of patient cohorts. }
%     \label{fig:ibPlot-vs-N-short}
% \end{figure}

%We next evaluate the effect of the lifetime of arms and asynchronous arrivals on the performance of the algorithms. 
In Figure~\ref{fig:ibPlot-vs-L-deterministic}, we consider an S-RMAB setting, in which arms continuously arrive according to a deterministic schedule, and leave after staying on for a lifetime of $L$, which we vary on the x-axis. The details about the other parameters are deferred to the appendix. We also study the isolated effects of small lifetimes and asynchronous arrivals separately as well as performance in settings with stochastic arrivals, in the appendix. Across the board, we find that the performance of \textsc{TW} degrades as the lifetime becomes shorter and that this effect only exacerbates with asynchronous arrivals. The performance of our algorithms remains on par with \citeauthor{qian2016restless}, in all of the above.

\begin{figure*}[t]
\begin{subfigure}{0.2\textwidth}
    \centering
    \includegraphics[width=0.89\linewidth]{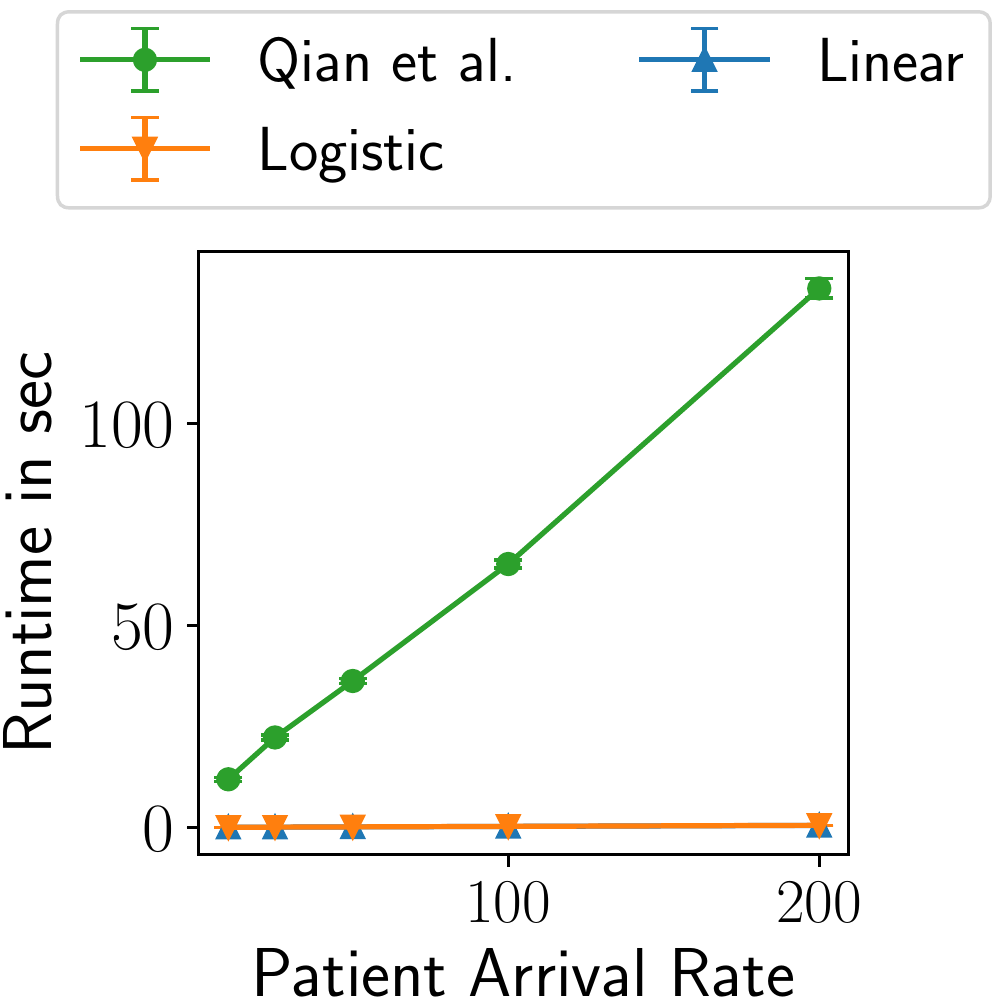}
    \caption{}
    \label{fig:runtime-vs-n-armman}
\end{subfigure}
\begin{subfigure}{0.31\textwidth}
    \centering
    \includegraphics[width=0.87\linewidth]{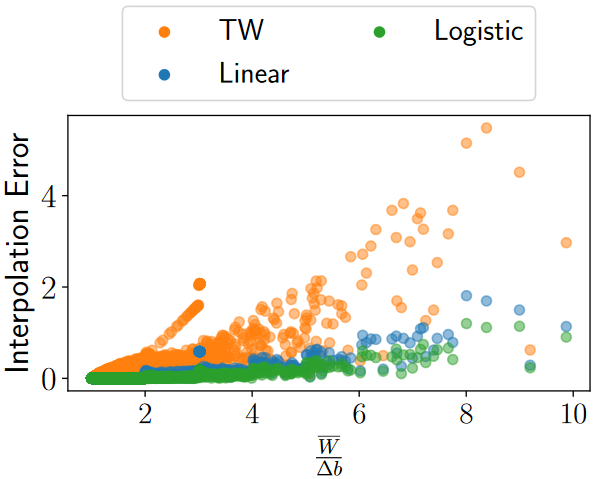}
    \caption{}
    \label{fig:error-scatterplot}
\end{subfigure}
\begin{subfigure}{0.36\textwidth}
    \centering
    \includegraphics[width=0.89\linewidth]{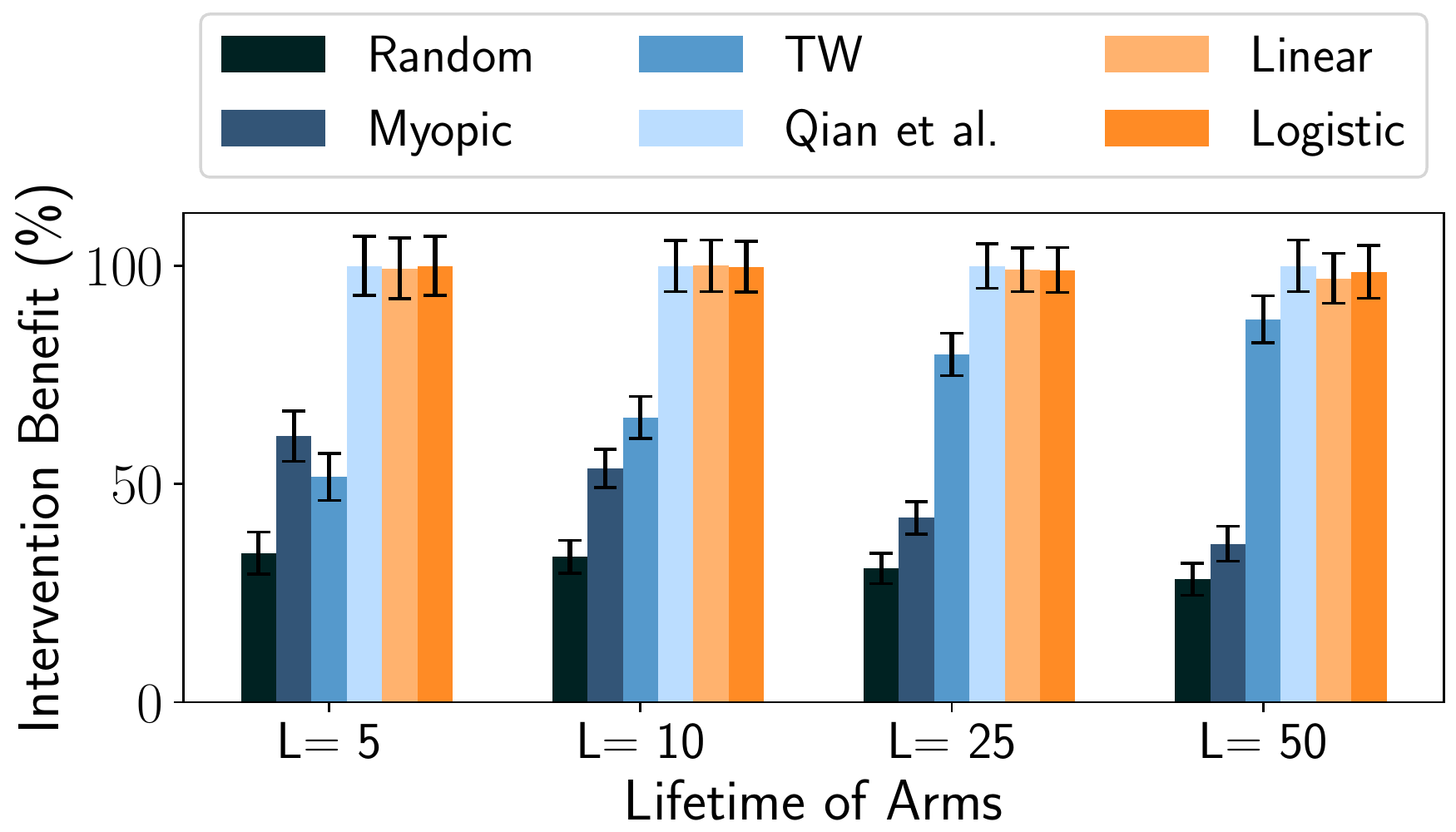}
    \caption{}
    \label{fig:ib-vs-L-reverse-threshold}
\end{subfigure}
\caption{(a) The interpolation algorithms achieve a speedup of about $250 \times$ over baselines.(b) The error between the actual and estimated indices is largest for TW and lower for our interpolation algorithms (c) The good performance is maintained even for reverse threshold optimal arms.}
\end{figure*}

\subsection{Real domain: ARMMAN for improving maternal healthcare}
Considering an alternate real-world domain, we again only use summary statistics (transition probabilities) from an application domain consisting of intervention planning for improving maternal healthcare \cite{biswas2021learn}. %This system involves a free-call based service offered by an India-based non-profit organization (ARMMAN) that aims to provide important preventive care information to new and expectant mothers. 
Individuals (arms) are labeled to be in one of three states at any time step, of which one is the good state. \cite{mate2021field} cast the problem as an RMAB with 2-state MDP on each arm. We also focus on maximizing the number of individuals in the good state, merging the other two states from the data into a single bad state. The data set consists of three types of transition matrices for different groups, only one of which satisfies the constraints mentioned in Section~\ref{sec:formulation} and is used in our subsequent analysis, which is otherwise analogous to Section~\ref{sec:TB_domain}. Figure~\ref{fig:runtime-vs-n-armman} establishes similar large runtime gains achieved by our algorithm as against other baselines, while maintaining similar performance figures in this domain. In the supplementary material we also present more details and analyses of the performance of our algorithms and baselines for this domain.

%depending on the fraction of times they receive calls. 

%\hl{TODO} The arms were further categorized into three groups, with each category having different state transition dynamics. We merge the intermediate and bad states to correspond to the bad state in our model and analogous computations as in Section~\ref{sec:TB_domain}.

\subsection{Synthetic domains}

Finally, in this section, we test our algorithms on synthetic domains. We identify corner cases where our solutions do poorly and construct adversarial domains based on those. The ratio between the infinite horizon Whittle Index $\bar{W}$ and the myopic index $\Delta b$ is an important driver of the approximation quality of our algorithms. The linear interpolation takes $\frac{\bar{W}}{\Delta b}$ steps to reach the finite horizon value, hence the higher this ratio is, higher the potential for approximation errors. In figure ~\ref{fig:error-scatterplot} we sum the approximation error over this interval $\epsilon \coloneqq \sum_{h=1}^{h=\frac{\overline{W}}{\Delta b}}(\|\hat{W}(h)- W_{Qian}(h)\|)$ and plot it for different ratios $\frac{\bar{W}}{\Delta b}$. As expected, the approximation error increases with $\frac{\bar{W}}{\Delta b}$. We construct an adversarial domain by simulating cohorts with varying proportions of such patients. The results in the supplementary material show the intervention benefit of our algorithms decreases but remains within one standard error of~\citeauthor{qian2016restless}

In Figure~\ref{fig:ib-vs-L-reverse-threshold}, we simulate a population consisting of reverse threshold optimal patients exclusively and show similar good performance even though the previous theoretical guarantees of Threshold Whittle apply to forward threshold optimal patients only. In the supplementary material, we test multiple synthetic domains by varying the proportion of forward threshold optimal patients. In addition, we perform several other robustness checks varying important problem parameters and find that the run time and strong performance of our algorithms remains consistent across the board.

\section{Conclusion}

%Finite horizon RMABs are an important class of bandits for many application settings such as in healthcare, where heterogeneous patients continuously enter and leave intervention program s. Inspired by such applications, w
We study \emph{streaming bandits}, or S-RMAB, a class of bandits where heterogeneous arms arrive and leave asynchronously under possibly random streams. While efficient RMAB algorithms for computing Whittle Indices for infinite horizon settings exist, for the finite horizon settings however, these algorithms %\citep{qian2016restless} 
are either comparatively costly or not suitable for estimating the Whittle Indices accurately. To tackle this, we provide a new scalable approach that allows for efficient computation of the Whittle Index values for finite horizon restless bandits while also adapting to more general S-RMAB settings. Our approach leverages a phenomenon called \textit{index decay} to compute the indices for each arm. Through an extensive set of experiments on real-world and synthetic data, we demonstrate that our approach provides good estimates of Whittle Indices, and yield over $200\times$ runtime improvements without loss in performance.

\begin{acks}
This work was supported in part by the Army Research Office by MURI grant number W911NF1810208. A.B. and C.S. were supported by the Harvard Center for Research on Computation and Society.
\end{acks}

%%%%%%%%%%%%%%%%%%%%%%%%%%%%%%%%%%%%%%%%%%%%%%%%%%%%%%%%%%%%%%%%%%%%%%%%

%%% The next two lines define, first, the bibliography style to be 
%%% applied, and, second, the bibliography file to be used.
%\newpage
%\balance

\balance
\bibliographystyle{ACM-Reference-Format} 
\bibliography{ref}

%%% -*-BibTeX-*-
%%% Do NOT edit. File created by BibTeX with style
%%% ACM-Reference-Format-Journals [18-Jan-2012].

\begin{thebibliography}{34}

%%% ====================================================================
%%% NOTE TO THE USER: you can override these defaults by providing
%%% customized versions of any of these macros before the \bibliography
%%% command.  Each of them MUST provide its own final punctuation,
%%% except for \shownote{}, \showDOI{}, and \showURL{}.  The latter two
%%% do not use final punctuation, in order to avoid confusing it with
%%% the Web address.
%%%
%%% To suppress output of a particular field, define its macro to expand
%%% to an empty string, or better, \unskip, like this:
%%%
%%% \newcommand{\showDOI}[1]{\unskip}   % LaTeX syntax
%%%
%%% \def \showDOI #1{\unskip}           % plain TeX syntax
%%%
%%% ====================================================================

\ifx \showCODEN    \undefined \def \showCODEN     #1{\unskip}     \fi
\ifx \showDOI      \undefined \def \showDOI       #1{#1}\fi
\ifx \showISBNx    \undefined \def \showISBNx     #1{\unskip}     \fi
\ifx \showISBNxiii \undefined \def \showISBNxiii  #1{\unskip}     \fi
\ifx \showISSN     \undefined \def \showISSN      #1{\unskip}     \fi
\ifx \showLCCN     \undefined \def \showLCCN      #1{\unskip}     \fi
\ifx \shownote     \undefined \def \shownote      #1{#1}          \fi
\ifx \showarticletitle \undefined \def \showarticletitle #1{#1}   \fi
\ifx \showURL      \undefined \def \showURL       {\relax}        \fi
% The following commands are used for tagged output and should be
% invisible to TeX
\providecommand\bibfield[2]{#2}
\providecommand\bibinfo[2]{#2}
\providecommand\natexlab[1]{#1}
\providecommand\showeprint[2][]{arXiv:#2}

\bibitem[\protect\citeauthoryear{Akbarzadeh and Mahajan}{Akbarzadeh and
  Mahajan}{2019}]%
        {Akbarzadeh2019}
\bibfield{author}{\bibinfo{person}{N. Akbarzadeh} {and} \bibinfo{person}{A.
  Mahajan}.} \bibinfo{year}{2019}\natexlab{}.
\newblock \showarticletitle{Restless bandits with controlled restarts:
  Indexability and computation of Whittle index}. In
  \bibinfo{booktitle}{\emph{2019 IEEE Conference on Decision and Control}}.
  IEEE.
\newblock


\bibitem[\protect\citeauthoryear{Bhattacharya}{Bhattacharya}{2018}]%
        {bhattacharya2018restless}
\bibfield{author}{\bibinfo{person}{Biswarup Bhattacharya}.}
  \bibinfo{year}{2018}\natexlab{}.
\newblock \showarticletitle{Restless bandits visiting villages: A preliminary
  study on distributing public health services}. In
  \bibinfo{booktitle}{\emph{Proceedings of the 1st ACM SIGCAS Conference on
  Computing and Sustainable Societies}}. \bibinfo{pages}{1--8}.
\newblock


\bibitem[\protect\citeauthoryear{Biswas, Aggarwal, Varakantham, and
  Tambe}{Biswas et~al\mbox{.}}{2021}]%
        {biswas2021learn}
\bibfield{author}{\bibinfo{person}{Arpita Biswas}, \bibinfo{person}{Gaurav
  Aggarwal}, \bibinfo{person}{Pradeep Varakantham}, {and}
  \bibinfo{person}{Milind Tambe}.} \bibinfo{year}{2021}\natexlab{}.
\newblock \showarticletitle{Learn to Intervene: An Adaptive Learning Policy for
  Restless Bandits in Application to Preventive Healthcare}. In
  \bibinfo{booktitle}{\emph{Proceedings of the 30th International Joint
  Conference on Artificial Intelligence}}.
\newblock


\bibitem[\protect\citeauthoryear{Biswas, Jain, Mandal, and Narahari}{Biswas
  et~al\mbox{.}}{2015}]%
        {biswas2015truthful}
\bibfield{author}{\bibinfo{person}{Arpita Biswas}, \bibinfo{person}{Shweta
  Jain}, \bibinfo{person}{Debmalya Mandal}, {and} \bibinfo{person}{Y
  Narahari}.} \bibinfo{year}{2015}\natexlab{}.
\newblock \showarticletitle{A truthful budget feasible multi-armed bandit
  mechanism for crowdsourcing time critical tasks}. In
  \bibinfo{booktitle}{\emph{Proceedings of the 2015 International Conference on
  Autonomous Agents and Multiagent Systems}}. \bibinfo{pages}{1101--1109}.
\newblock


\bibitem[\protect\citeauthoryear{Brownstein, Chowdhury, Norris, Horsley,
  Jack~Jr, Zhang, and Satterfield}{Brownstein et~al\mbox{.}}{2007}]%
        {brownstein2007}
\bibfield{author}{\bibinfo{person}{J~Nell Brownstein}, \bibinfo{person}{Farah~M
  Chowdhury}, \bibinfo{person}{Susan~L Norris}, \bibinfo{person}{Tanya
  Horsley}, \bibinfo{person}{Leonard Jack~Jr}, \bibinfo{person}{Xuanping
  Zhang}, {and} \bibinfo{person}{Dawn Satterfield}.}
  \bibinfo{year}{2007}\natexlab{}.
\newblock \showarticletitle{Effectiveness of community health workers in the
  care of people with hypertension}.
\newblock \bibinfo{journal}{\emph{American journal of preventive medicine}}
  \bibinfo{volume}{32}, \bibinfo{number}{5} (\bibinfo{year}{2007}),
  \bibinfo{pages}{435--447}.
\newblock


\bibitem[\protect\citeauthoryear{Chang, Polesky, and Bhatia}{Chang
  et~al\mbox{.}}{2013}]%
        {chang2013}
\bibfield{author}{\bibinfo{person}{Alicia~H Chang}, \bibinfo{person}{Andrea
  Polesky}, {and} \bibinfo{person}{Gulshan Bhatia}.}
  \bibinfo{year}{2013}\natexlab{}.
\newblock \showarticletitle{House calls by community health workers and public
  health nurses to improve adherence to isoniazid monotherapy for latent
  tuberculosis infection: a retrospective study}.
\newblock \bibinfo{journal}{\emph{BMC public health}} \bibinfo{volume}{13},
  \bibinfo{number}{1} (\bibinfo{year}{2013}), \bibinfo{pages}{894}.
\newblock


\bibitem[\protect\citeauthoryear{Glazebrook, Ruiz-Hernandex, and
  Kirkbride}{Glazebrook et~al\mbox{.}}{2006}]%
        {Glazebrook2006}
\bibfield{author}{\bibinfo{person}{K.D. Glazebrook}, \bibinfo{person}{D.
  Ruiz-Hernandex}, {and} \bibinfo{person}{C. Kirkbride}.}
  \bibinfo{year}{2006}\natexlab{}.
\newblock \showarticletitle{Some indexable families of restless bandit
  problems}.
\newblock \bibinfo{journal}{\emph{Adv. Appl. Probab}} (\bibinfo{year}{2006}),
  \bibinfo{pages}{643--672}.
\newblock


\bibitem[\protect\citeauthoryear{Hawkins}{Hawkins}{2003}]%
        {hawkins2003langrangian}
\bibfield{author}{\bibinfo{person}{Jeffrey~Thomas Hawkins}.}
  \bibinfo{year}{2003}\natexlab{}.
\newblock \emph{\bibinfo{title}{A Langrangian decomposition approach to weakly
  coupled dynamic optimization problems and its applications}}.
\newblock \bibinfo{thesistype}{Ph.D. Dissertation}.
  \bibinfo{school}{Massachusetts Institute of Technology}.
\newblock


\bibitem[\protect\citeauthoryear{Hsu}{Hsu}{2018}]%
        {Hsu2018}
\bibfield{author}{\bibinfo{person}{Y. Hsu}.} \bibinfo{year}{2018}\natexlab{}.
\newblock \showarticletitle{Age of information: Whittle index for scheduling
  stochastic arrivals}. In \bibinfo{booktitle}{\emph{2018 IEEE International
  Symposium on Information Theory}}. IEEE.
\newblock


\bibitem[\protect\citeauthoryear{Hu and Frazier}{Hu and Frazier}{2017}]%
        {hu2017asymptotically}
\bibfield{author}{\bibinfo{person}{Weici Hu} {and} \bibinfo{person}{Peter
  Frazier}.} \bibinfo{year}{2017}\natexlab{}.
\newblock \showarticletitle{An asymptotically optimal index policy for
  finite-horizon restless bandits}.
\newblock \bibinfo{journal}{\emph{arXiv preprint arXiv:1707.00205}}
  (\bibinfo{year}{2017}).
\newblock


\bibitem[\protect\citeauthoryear{Jiang, Song, Weinstein, and Zhang}{Jiang
  et~al\mbox{.}}{2020}]%
        {jiang2020faster}
\bibfield{author}{\bibinfo{person}{S. Jiang}, \bibinfo{person}{Z. Song},
  \bibinfo{person}{O. Weinstein}, {and} \bibinfo{person}{H. Zhang}.}
  \bibinfo{year}{2020}\natexlab{}.
\newblock \showarticletitle{Faster dynamic matrix inverse for faster lps}.
\newblock \bibinfo{journal}{\emph{arXiv preprint arXiv:2004.07470}}
  (\bibinfo{year}{2020}).
\newblock


\bibitem[\protect\citeauthoryear{Kanade, McMahan, and Bryan}{Kanade
  et~al\mbox{.}}{2009}]%
        {kanade2009sleeping}
\bibfield{author}{\bibinfo{person}{Varun Kanade}, \bibinfo{person}{H~Brendan
  McMahan}, {and} \bibinfo{person}{Brent Bryan}.}
  \bibinfo{year}{2009}\natexlab{}.
\newblock \showarticletitle{Sleeping experts and bandits with stochastic action
  availability and adversarial rewards}. In
  \bibinfo{booktitle}{\emph{Artificial Intelligence and Statistics}}. PMLR,
  \bibinfo{pages}{272--279}.
\newblock


\bibitem[\protect\citeauthoryear{Killian, Wilder, Sharma, Choudhary, Dilkina,
  and Tambe}{Killian et~al\mbox{.}}{2019}]%
        {killian2019learning}
\bibfield{author}{\bibinfo{person}{J.~A. Killian}, \bibinfo{person}{B. Wilder},
  \bibinfo{person}{A. Sharma}, \bibinfo{person}{V. Choudhary},
  \bibinfo{person}{B. Dilkina}, {and} \bibinfo{person}{M. Tambe}.}
  \bibinfo{year}{2019}\natexlab{}.
\newblock \showarticletitle{Learning to Prescribe Interventions for
  Tuberculosis Patients using Digital Adherence Data}. In
  \bibinfo{booktitle}{\emph{KDD}}.
\newblock


\bibitem[\protect\citeauthoryear{Kleinberg, Niculescu-Mizil, and
  Sharma}{Kleinberg et~al\mbox{.}}{2010}]%
        {kleinberg2010regret}
\bibfield{author}{\bibinfo{person}{Robert Kleinberg},
  \bibinfo{person}{Alexandru Niculescu-Mizil}, {and} \bibinfo{person}{Yogeshwer
  Sharma}.} \bibinfo{year}{2010}\natexlab{}.
\newblock \showarticletitle{Regret bounds for sleeping experts and bandits}.
\newblock \bibinfo{journal}{\emph{Machine learning}} \bibinfo{volume}{80},
  \bibinfo{number}{2} (\bibinfo{year}{2010}), \bibinfo{pages}{245--272}.
\newblock


\bibitem[\protect\citeauthoryear{Le~Ny, Dahleh, and Feron}{Le~Ny
  et~al\mbox{.}}{2008}]%
        {le2008multi}
\bibfield{author}{\bibinfo{person}{Jerome Le~Ny}, \bibinfo{person}{Munther
  Dahleh}, {and} \bibinfo{person}{Eric Feron}.}
  \bibinfo{year}{2008}\natexlab{}.
\newblock \showarticletitle{Multi-UAV dynamic routing with partial observations
  using restless bandit allocation indices}. In \bibinfo{booktitle}{\emph{2008
  American Control Conference}}. IEEE, \bibinfo{pages}{4220--4225}.
\newblock


\bibitem[\protect\citeauthoryear{Lee, Lavieri, and Volk}{Lee
  et~al\mbox{.}}{2019}]%
        {lee2019optimal}
\bibfield{author}{\bibinfo{person}{Elliot Lee}, \bibinfo{person}{Mariel~S
  Lavieri}, {and} \bibinfo{person}{Michael Volk}.}
  \bibinfo{year}{2019}\natexlab{}.
\newblock \showarticletitle{Optimal screening for hepatocellular carcinoma: A
  restless bandit model}.
\newblock \bibinfo{journal}{\emph{Manufacturing \& Service Operations
  Management}} \bibinfo{volume}{21}, \bibinfo{number}{1}
  (\bibinfo{year}{2019}), \bibinfo{pages}{198--212}.
\newblock


\bibitem[\protect\citeauthoryear{Little}{Little}{1961}]%
        {little1961proof}
\bibfield{author}{\bibinfo{person}{John~DC Little}.}
  \bibinfo{year}{1961}\natexlab{}.
\newblock \showarticletitle{A proof for the queuing formula: L= $\lambda$ W}.
\newblock \bibinfo{journal}{\emph{Operations research}} \bibinfo{volume}{9},
  \bibinfo{number}{3} (\bibinfo{year}{1961}), \bibinfo{pages}{383--387}.
\newblock


\bibitem[\protect\citeauthoryear{Liu and Zhao}{Liu and Zhao}{2010a}]%
        {zhao_liyu_paper}
\bibfield{author}{\bibinfo{person}{K. Liu} {and} \bibinfo{person}{Q. Zhao}.}
  \bibinfo{year}{2010}\natexlab{a}.
\newblock \showarticletitle{Indexability of restless bandit problems and
  optimality of {W}hittle index for dynamic multichannel access}.
\newblock \bibinfo{journal}{\emph{IEEE Transactions on Information Theory}}
  \bibinfo{volume}{56}, \bibinfo{number}{11} (\bibinfo{year}{2010}),
  \bibinfo{pages}{5547--5567}.
\newblock


\bibitem[\protect\citeauthoryear{Liu and Zhao}{Liu and Zhao}{2010b}]%
        {Liu2010}
\bibfield{author}{\bibinfo{person}{K. Liu} {and} \bibinfo{person}{Q. Zhao}.}
  \bibinfo{year}{2010}\natexlab{b}.
\newblock \showarticletitle{Indexability of restless bandit problems and
  optimality of Whittle index for dynamic multichannel access}.
\newblock \bibinfo{journal}{\emph{IEEE Transactions on Information Theory}}
  (\bibinfo{year}{2010}), \bibinfo{pages}{5547--5567}.
\newblock


\bibitem[\protect\citeauthoryear{L{\"o}we, Un{\"u}tzer, Callahan, Perkins, and
  Kroenke}{L{\"o}we et~al\mbox{.}}{2004}]%
        {lowe2004}
\bibfield{author}{\bibinfo{person}{Bernd L{\"o}we}, \bibinfo{person}{J{\"u}rgen
  Un{\"u}tzer}, \bibinfo{person}{Christopher~M Callahan},
  \bibinfo{person}{Anthony~J Perkins}, {and} \bibinfo{person}{Kurt Kroenke}.}
  \bibinfo{year}{2004}\natexlab{}.
\newblock \showarticletitle{Monitoring depression treatment outcomes with the
  patient health questionnaire-9}.
\newblock \bibinfo{journal}{\emph{Medical care}} (\bibinfo{year}{2004}),
  \bibinfo{pages}{1194--1201}.
\newblock


\bibitem[\protect\citeauthoryear{Mate, Killian, Xu, Perrault, and Tambe}{Mate
  et~al\mbox{.}}{2020}]%
        {mate2020collapsing}
\bibfield{author}{\bibinfo{person}{Aditya Mate}, \bibinfo{person}{Jackson~A
  Killian}, \bibinfo{person}{Haifeng Xu}, \bibinfo{person}{Andrew Perrault},
  {and} \bibinfo{person}{Milind Tambe}.} \bibinfo{year}{2020}\natexlab{}.
\newblock \showarticletitle{Collapsing Bandits and Their Application to Public
  Health Interventions}. In \bibinfo{booktitle}{\emph{Advances in Neural and
  Information Processing Systems (NeurIPS)}}.
\newblock


\bibitem[\protect\citeauthoryear{Mate, Madaan, Taneja, Madhiwalla, Verma,
  Singh, Hegde, Varakantham, and Tambe}{Mate et~al\mbox{.}}{2021a}]%
        {mate2021field}
\bibfield{author}{\bibinfo{person}{Aditya Mate}, \bibinfo{person}{Lovish
  Madaan}, \bibinfo{person}{Aparna Taneja}, \bibinfo{person}{Neha Madhiwalla},
  \bibinfo{person}{Shresth Verma}, \bibinfo{person}{Gargi Singh},
  \bibinfo{person}{Aparna Hegde}, \bibinfo{person}{Pradeep Varakantham}, {and}
  \bibinfo{person}{Milind Tambe}.} \bibinfo{year}{2021}\natexlab{a}.
\newblock \showarticletitle{Field Study in Deploying Restless Multi-Armed
  Bandits: Assisting Non-Profits in Improving Maternal and Child Health}.
\newblock \bibinfo{journal}{\emph{arXiv preprint arXiv:2109.08075}}
  (\bibinfo{year}{2021}).
\newblock


\bibitem[\protect\citeauthoryear{Mate, Perrault, and Tambe}{Mate
  et~al\mbox{.}}{2021b}]%
        {mate2021riskAware}
\bibfield{author}{\bibinfo{person}{Aditya Mate}, \bibinfo{person}{Andrew
  Perrault}, {and} \bibinfo{person}{Milind Tambe}.}
  \bibinfo{year}{2021}\natexlab{b}.
\newblock \showarticletitle{Risk-Aware Interventions in Public Health: Planning
  with Restless Multi-Armed Bandits}. In
  \bibinfo{booktitle}{\emph{International Conference on Autonomous Agents and
  Multiagent Systems}}.
\newblock


\bibitem[\protect\citeauthoryear{Meuleau, Hauskrecht, Kim, Peshkin, Kaelbling,
  Dean, and Boutilier}{Meuleau et~al\mbox{.}}{1998}]%
        {meuleau1998solving}
\bibfield{author}{\bibinfo{person}{Nicolas Meuleau}, \bibinfo{person}{Milos
  Hauskrecht}, \bibinfo{person}{Kee-Eung Kim}, \bibinfo{person}{Leonid
  Peshkin}, \bibinfo{person}{Leslie~Pack Kaelbling}, \bibinfo{person}{Thomas~L
  Dean}, {and} \bibinfo{person}{Craig Boutilier}.}
  \bibinfo{year}{1998}\natexlab{}.
\newblock \showarticletitle{Solving very large weakly coupled Markov decision
  processes}. In \bibinfo{booktitle}{\emph{AAAI/IAAI}}.
  \bibinfo{pages}{165--172}.
\newblock


\bibitem[\protect\citeauthoryear{Mundorf, Shankar, Moran, Heller, Hassan,
  Harville, and Lichtveld}{Mundorf et~al\mbox{.}}{2018}]%
        {mundorf2018}
\bibfield{author}{\bibinfo{person}{Christopher Mundorf}, \bibinfo{person}{Arti
  Shankar}, \bibinfo{person}{Tracy Moran}, \bibinfo{person}{Sherry Heller},
  \bibinfo{person}{Anna Hassan}, \bibinfo{person}{Emily Harville}, {and}
  \bibinfo{person}{Maureen Lichtveld}.} \bibinfo{year}{2018}\natexlab{}.
\newblock \showarticletitle{Reducing the risk of postpartum depression in a
  low-income community through a community health worker intervention}.
\newblock \bibinfo{journal}{\emph{Maternal and child health journal}}
  \bibinfo{volume}{22}, \bibinfo{number}{4} (\bibinfo{year}{2018}),
  \bibinfo{pages}{520--528}.
\newblock


\bibitem[\protect\citeauthoryear{Newman, Franke, Arrieta, Carrasco, Elliott,
  Flores, Friedman, Graham, Martinez, Palazuelos, et~al\mbox{.}}{Newman
  et~al\mbox{.}}{2018}]%
        {newman2018}
\bibfield{author}{\bibinfo{person}{Patrick~M Newman}, \bibinfo{person}{Molly~F
  Franke}, \bibinfo{person}{Jafet Arrieta}, \bibinfo{person}{Hector Carrasco},
  \bibinfo{person}{Patrick Elliott}, \bibinfo{person}{Hugo Flores},
  \bibinfo{person}{Alexandra Friedman}, \bibinfo{person}{Sophia Graham},
  \bibinfo{person}{Luis Martinez}, \bibinfo{person}{Lindsay Palazuelos},
  {et~al\mbox{.}}} \bibinfo{year}{2018}\natexlab{}.
\newblock \showarticletitle{Community health workers improve disease control
  and medication adherence among patients with diabetes and/or hypertension in
  Chiapas, Mexico: an observational stepped-wedge study}.
\newblock \bibinfo{journal}{\emph{BMJ Global Health}} (\bibinfo{year}{2018}).
\newblock


\bibitem[\protect\citeauthoryear{Nino-Mora}{Nino-Mora}{2011}]%
        {nino2011computing}
\bibfield{author}{\bibinfo{person}{Jos{\'e} Nino-Mora}.}
  \bibinfo{year}{2011}\natexlab{}.
\newblock \showarticletitle{Computing a classic index for finite-horizon
  bandits}.
\newblock \bibinfo{journal}{\emph{INFORMS Journal on Computing}}
  \bibinfo{volume}{23}, \bibinfo{number}{2} (\bibinfo{year}{2011}),
  \bibinfo{pages}{254--267}.
\newblock


\bibitem[\protect\citeauthoryear{Papadimitriou and Tsitsiklis}{Papadimitriou
  and Tsitsiklis}{1994}]%
        {papadimitriou1994complexity}
\bibfield{author}{\bibinfo{person}{Christos~H Papadimitriou} {and}
  \bibinfo{person}{John~N Tsitsiklis}.} \bibinfo{year}{1994}\natexlab{}.
\newblock \showarticletitle{The complexity of optimal queueing network
  control}. In \bibinfo{booktitle}{\emph{Proceedings of IEEE 9th Annual
  Conference on Structure in Complexity Theory}}. IEEE,
  \bibinfo{pages}{318--322}.
\newblock


\bibitem[\protect\citeauthoryear{Qian, Zhang, Krishnamachari, and Tambe}{Qian
  et~al\mbox{.}}{2016}]%
        {qian2016restless}
\bibfield{author}{\bibinfo{person}{Y. Qian}, \bibinfo{person}{C. Zhang},
  \bibinfo{person}{B. Krishnamachari}, {and} \bibinfo{person}{M. Tambe}.}
  \bibinfo{year}{2016}\natexlab{}.
\newblock \showarticletitle{Restless poachers: Handling
  exploration-exploitation tradeoffs in security domains}. In
  \bibinfo{booktitle}{\emph{AAMAS}}.
\newblock


\bibitem[\protect\citeauthoryear{Rahedi~Ong'ang'o, Mwachari, Kipruto, and
  Karanja}{Rahedi~Ong'ang'o et~al\mbox{.}}{2014}]%
        {rahedi2014}
\bibfield{author}{\bibinfo{person}{Jane Rahedi~Ong'ang'o},
  \bibinfo{person}{Christina Mwachari}, \bibinfo{person}{Hillary Kipruto},
  {and} \bibinfo{person}{Simon Karanja}.} \bibinfo{year}{2014}\natexlab{}.
\newblock \showarticletitle{The effects on tuberculosis treatment adherence
  from utilising community health workers: a comparison of selected rural and
  urban settings in Kenya}.
\newblock \bibinfo{journal}{\emph{PLoS One}} \bibinfo{volume}{9},
  \bibinfo{number}{2} (\bibinfo{year}{2014}), \bibinfo{pages}{e88937}.
\newblock


\bibitem[\protect\citeauthoryear{Sombabu, Mate, Manjunath, and Moharir}{Sombabu
  et~al\mbox{.}}{2020}]%
        {Sombabu2020}
\bibfield{author}{\bibinfo{person}{B. Sombabu}, \bibinfo{person}{A. Mate},
  \bibinfo{person}{D. Manjunath}, {and} \bibinfo{person}{S. Moharir}.}
  \bibinfo{year}{2020}\natexlab{}.
\newblock \showarticletitle{Whittle index for AoI-aware scheduling}. In
  \bibinfo{booktitle}{\emph{IEEE International Conference on Communication
  Systems \& Networks (COMSNETS)}}. IEEE.
\newblock


\bibitem[\protect\citeauthoryear{Sondik}{Sondik}{1978}]%
        {sondik1978optimal}
\bibfield{author}{\bibinfo{person}{Edward~J Sondik}.}
  \bibinfo{year}{1978}\natexlab{}.
\newblock \showarticletitle{The optimal control of partially observable Markov
  processes over the infinite horizon: Discounted costs}.
\newblock \bibinfo{journal}{\emph{Operations research}} \bibinfo{volume}{26},
  \bibinfo{number}{2} (\bibinfo{year}{1978}), \bibinfo{pages}{282--304}.
\newblock


\bibitem[\protect\citeauthoryear{Whittle}{Whittle}{1988}]%
        {whittle1988restless}
\bibfield{author}{\bibinfo{person}{P. Whittle}.}
  \bibinfo{year}{1988}\natexlab{}.
\newblock \showarticletitle{Restless bandits: Activity allocation in a changing
  world}.
\newblock \bibinfo{journal}{\emph{J. Appl. Probab.}} \bibinfo{volume}{25},
  \bibinfo{number}{A} (\bibinfo{year}{1988}), \bibinfo{pages}{287--298}.
\newblock


\bibitem[\protect\citeauthoryear{Zayas-Caban, Jasin, and Wang}{Zayas-Caban
  et~al\mbox{.}}{2019}]%
        {zayas2019asymptotically}
\bibfield{author}{\bibinfo{person}{Gabriel Zayas-Caban},
  \bibinfo{person}{Stefanus Jasin}, {and} \bibinfo{person}{Guihua Wang}.}
  \bibinfo{year}{2019}\natexlab{}.
\newblock \showarticletitle{An asymptotically optimal heuristic for general
  nonstationary finite-horizon restless multi-armed, multi-action bandits}.
\newblock \bibinfo{journal}{\emph{Advances in Applied Probability}}
  \bibinfo{volume}{51}, \bibinfo{number}{3} (\bibinfo{year}{2019}),
  \bibinfo{pages}{745--772}.
\newblock


\end{thebibliography}
%\nobalance

%\input{appendix}
%%%%%%%%%%%%%%%%%%%%%%%%%%%%%%%%%%%%%%%%%%%%%%%%%%%%%%%%%%%%%%%%%%%%%%%%

\end{document}